\documentclass[11pt]{article}

\PassOptionsToPackage{hyperfootnotes=false}{hyperref}
\usepackage[final]{acl}

\usepackage{times}
\usepackage{latexsym}
\usepackage[T1]{fontenc}
\usepackage[utf8]{inputenc}
\usepackage{microtype}
\usepackage{inconsolata}
\usepackage{graphicx}

\usepackage{enumitem}
\usepackage{pifont}
\usepackage{amsmath,amssymb,amsthm}
\usepackage{algorithm}
\usepackage{algorithmic}
\usepackage{booktabs}
\usepackage{multirow}
\usepackage{float}  

\newtheorem{theorem}{Theorem}[section]
\newtheorem{proposition}[theorem]{Proposition}
\newtheorem{lemma}[theorem]{Lemma}

\newtheorem{definition}[theorem]{Definition}
\newtheorem{assumption}{Assumption}
\newtheorem{remark}{Remark}[section]

\newtheorem{decomposition}[theorem]{Decomposition}

\title{DIANOIA: Diagnostic Decomposition and Joint Optimization for Multi-Agent Reasoning}

  \author{
    Yiming Yang\thanks{\,ORCID: 0000-0003-1359-0364} \and
    Zhuoyuan Li \and
    Fanxiang Zeng \and
    Hao Fu \and
    Yue Liu \\
    AMap, Alibaba Group, Beijing, China \\
    \texttt{\{sachiel.yym, weiyuan.lzy, fanxiang.zfx,} \\
    \texttt{fh265565, yue.liu\}@alibaba-inc.com}
  }

\begin{document}
\maketitle

\begin{abstract}
Multi-agent LLM systems can outperform single-agent baselines, yet practitioners still cannot predict which design works for a new task or diagnose why one fails.
We argue this gap persists largely because the field lacks a diagnostic framework with measurable primitives and testable predictions. We introduce \textbf{DIANOIA}, a three-channel decomposition of multi-agent reasoning gain into coverage, fidelity, and synthesis, each of which is empirically measurable. From this decomposition, we propose a diagnostic protocol that forms qualitative hypotheses about bottleneck channels for tasks with measurable quality signals. We instantiate the protocol as a multi-agent system whose three components mirror the channels: role-diverse proposers for coverage, execution- or pseudo-verification-grounded review for fidelity, and iterative synthesis. On GSM8K, MBPP, and BFCL-SP, our method outperforms strong multi-agent baselines under matched standard configurations; on MBPP, it matches the strongest baseline with approximately $5{\times}$ fewer tokens and dominates the measured accuracy--token Pareto frontier. Full-scale MBPP replications on Qwen, DeepSeek, and Gemini consistently improve over the corresponding single-model baseline. A prospective held-out evaluation on VitaBench, whose bottleneck hypothesis was recorded before running DIANOIA, improves pass rate from 29.5\% to 45.0\%. DIANOIA reframes multi-agent design as channel-aware resource allocation: diagnose the likely bottleneck, then invest tokens accordingly.
\end{abstract}

\section{Introduction} \label{sec:introduction}
Reasoning with Large Language Models (LLMs) has shifted from single-prompt Chain-of-Thought~\cite{wei2022chain} to systems that orchestrate multiple LLM instances as collaborative agents~\cite{guo2024large, tran2025multi}. Multi-Agent Systems (MAS) and related inference-time agent methods have produced substantial empirical gains in mathematical reasoning~\cite{wang2023self}, code generation~\cite{shinn2023reflexion}, and interactive decision-making~\cite{zhou2024lats}. Yet the mechanisms behind these gains remain poorly understood: in particular, it is unclear what drives improvement, when a given multi-agent design should help, and how practitioners should choose among methods for a new task.

Despite these empirical successes, existing methods remain largely heuristic: practitioners lack a principled basis for forming ex-ante hypotheses about which method may work on which task, or why a given method underperforms. We argue the path forward is twofold: (a) a \emph{diagnostic framework} with measurable primitives and testable predictions, and (b) an \emph{instantiated system} showing that the framework's prescriptions can be realized simultaneously in one system. These two goals are complementary: the framework organizes hypotheses about \emph{why} gains arise, while the system tests whether those gains can be achieved in practice. Its formal guarantees apply only under the verifier and independence conditions stated in Section~\ref{sec:theory}.

We introduce a diagnostic decomposition of multi-agent reasoning gain---defined as $\mathbb{E}[Q(\tau^{\text{MAS}})] - p$, the improvement in expected solution quality over a single-agent baseline---formalized below as an upper bound with three conceptually distinct, empirically measurable components (formalized in Section~\ref{sec:theory}, Decomposition~\ref{def:decomposition}):
\begin{equation}
\label{eq:decomposition_overview}
\mathbb{E}[Q(\tau^{\text{MAS}})]\!-\!p \;\leq\; \underbrace{\mathcal{G}_{\text{explore}}}_{\text{coverage}} + \underbrace{\mathcal{G}_{\text{info}}}_{\text{fidelity}} + \underbrace{\mathcal{G}_{\text{aggr}}}_{\text{synthesis}}
\end{equation}
where each component corresponds to a distinct source of gain in multi-agent reasoning:

\begin{itemize}[leftmargin=*]
    \item \textbf{Exploration Gain} ($\mathcal{G}_{\text{explore}}$): The benefit of \textit{seeing more}: covering a larger portion of the solution space through diverse proposals. This gain grows with the number of agents $K$ and with the diversity of their reasoning strategies, through the probability that at least one agent discovers a correct solution path.
    
    \item \textbf{Information Gain} ($\mathcal{G}_{\text{info}}$): The benefit of \textit{seeing clearly}: obtaining quality signals that reliably distinguish correct candidate solutions from incorrect ones. In our setting, such signals come from execution feedback $e$ (e.g., code test results or tool-call return values) or from evidence-based cross-review that converts raw observations into structured assessments. The key quantity is therefore not feedback alone, but the \emph{fidelity} of the information available for selection.
    
    \item \textbf{Aggregation Gain} ($\mathcal{G}_{\text{aggr}}$): The benefit of \textit{deciding wisely}: synthesizing diverse proposals and feedback into a strong final solution. This gain depends on how effectively the aggregation mechanism uses available information while avoiding failure modes such as groupthink or agreement bias~\cite{pitre2025consensagent}.
\end{itemize}

Guided by this decomposition, we propose \textbf{DIANOIA}\footnote{From Greek \emph{dianoia} (Plato): discursive reasoning.}, a four-phase multi-agent system (Propose, Execute, Review, Synthesize) designed to address all three sources of gain: role-diverse proposal generation for Exploration, execution-grounded verification---combining execution feedback with evidence-based cross-review---for Information, and iterative synthesis with closed-loop validation for Aggregation. The three roles are one lightweight realization of the formal reduced-correlation prescription.

We evaluate DIANOIA on four benchmarks spanning mathematical reasoning (GSM8K~\cite{cobbe2021training}, AIME-2025~\cite{aime2025}), code generation (MBPP~\cite{austin2021program}), and function calling (BFCL-SP~\cite{patil2025bfcl}). Across GSM8K, MBPP, and BFCL-SP, DIANOIA outperforms strong multi-agent and Best-of-$N$ baselines under matched standard configurations; we treat AIME-2025 as a small-sample stress test on its complete 30-problem set. On MBPP, it reaches the strongest baseline's accuracy ceiling using roughly $5{\times}$ fewer tokens and dominates the measured budget--accuracy Pareto frontier. Methods that improve only one or two dimensions tend to saturate earlier, whereas DIANOIA's joint design sustains accuracy gains over a broader budget range, consistent with the framework's predictions. We additionally conduct a prospective held-out diagnostic evaluation on VitaBench and full-scale MBPP replications on DeepSeek-V4-Flash and Gemini-3.5-Flash.

Our contributions are:
\begin{itemize}[leftmargin=*,topsep=2pt,itemsep=2pt]
    \item A \emph{prescriptive} diagnostic decomposition of multi-agent reasoning gain into coverage, information, and aggregation, with measurable estimators for each component. The decomposition places prior methods on a three-axis design map and yields a \emph{four-rule diagnostic protocol} (R1--R4, \S\ref{sec:method}) that forms qualitative bottleneck hypotheses for tasks with measurable quality signals.

    \item \textbf{DIANOIA}, a four-phase system that jointly enacts all three prescriptions of the decomposition. Each individual ingredient---role-diverse proposal, execution feedback, cross-review, and iterative synthesis---appears in prior work; our contribution is to combine them under a single diagnostic framework and evaluate their interaction systematically. Phase- and role-level ablations, together with a Best-of-$N$ control, show that no tested proper subset recovers the full gain.

    \item Budget-matched comparisons on GSM8K, MBPP, and BFCL-SP---together with a small-$N$ AIME-2025 stress test, a prospective held-out VitaBench evaluation, and full-scale MBPP replications across Qwen, DeepSeek, and Gemini. A full Pareto-frontier sweep on MBPP shows that DIANOIA matches the strongest baseline with approximately $5{\times}$ fewer tokens and improves the measured accuracy--cost frontier over a broad range of token budgets. The framework's bottleneck hypotheses are consistent with the observed gain patterns (Section~\ref{subsec:discussion}).
\end{itemize}

\begin{figure*}[t]
    \centering
    \includegraphics[width=\textwidth]{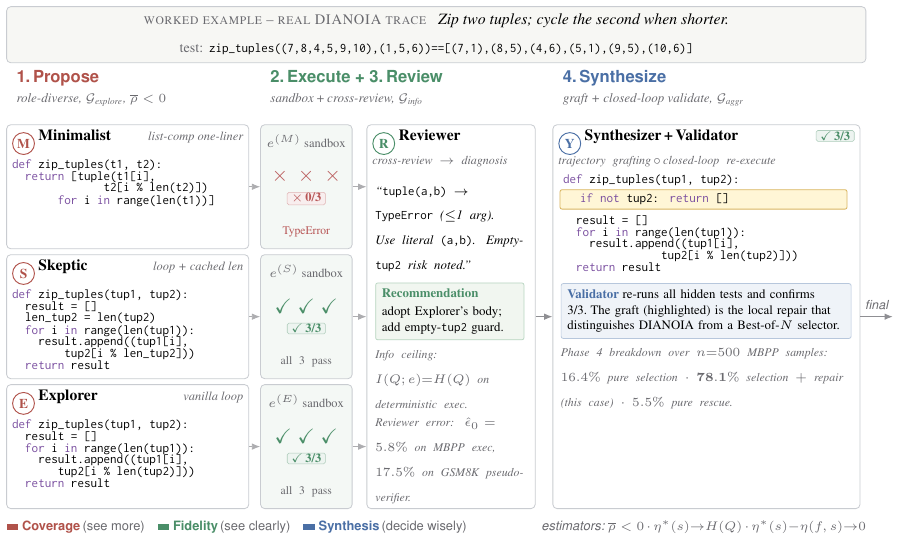}
    \caption{A DIANOIA trace on MBPP. Full log in Appendix~\ref{app:figure_one_raw}.}
    \label{fig:overview}
\end{figure*}

\section{Related Work}
\label{sec:related_work}

We organize related work through the lens of our gain decomposition (Eq.~\ref{eq:decomposition_overview}; formalized in Section~\ref{sec:theory}), which highlights three recurring design emphases in prior work: Exploration, Information, and Aggregation. Existing methods often prioritize one or two of these dimensions, which helps explain their complementary strengths and limitations and motivates our unified treatment.

\textbf{Diversity and Exploration ($\mathcal{G}_{\text{explore}}$).} Self-Consistency~\cite{wang2023self} samples multiple reasoning paths and selects via majority vote; Tree of Thoughts~\cite{yao2023tree} structures exploration as tree search; diversified sampling~\cite{naik2023diversity,wang2025diversified} promotes explicit path diversity beyond temperature; and ReConcile~\cite{chen2024reconcile} orchestrates round-table discussions over heterogeneous LLMs. These methods primarily improve $\mathcal{G}_{\text{explore}}$ by broadening coverage of the solution space. However, many rely on limited or purely internal quality signals, and their aggregation mechanisms can remain vulnerable when errors are correlated across samples or agents.

\textbf{Grounding and Feedback ($\mathcal{G}_{\text{info}}$).} ReAct~\cite{yao2023react} interleaves reasoning with tool actions; Reflexion~\cite{shinn2023reflexion} reflects on execution feedback; LATS~\cite{zhou2024lats} combines search with environmental verifiers; AgentCoder~\cite{huang2024agentcoder} combines a Programmer agent with a Test-Executor agent for iterative test-driven repair; step-level process reward models~\cite{lightman2023lets,setlur2024rewarding} provide denser supervision signals; and collaborative verification~\cite{liang2024improving} strengthens test-time quality assessment. These approaches improve $\mathcal{G}_{\text{info}}$ by strengthening the fidelity of intermediate or final feedback. However, they typically explore alternatives sequentially within a single search process rather than through parallel, heterogeneous proposal generation, which can limit $\mathcal{G}_{\text{explore}}$.

\textbf{Interaction and Consensus ($\mathcal{G}_{\text{aggr}}$).} Multi-Agent Debate~\cite{du2024improving} lets agents critique each other; Encouraging Divergent Thinking~\cite{liang2024encouraging} promotes early disagreement; Mixture-of-Agents (MoA)~\cite{wang2024mixture} stacks aggregation layers; Two Heads~\cite{jin2025two} studies multi-agent test-time scaling, echoing single-agent compute-scaling results~\cite{snell2024scaling} that also motivate the Best-of-$N$ control we use. Pure debate is ``cheap talk'' without evidence; ConsensAgent~\cite{pitre2025consensagent} documents sycophancy and agreement bias under purely textual interaction. Surveys~\cite{guo2024large,tran2025multi,ke2025survey} catalogue the landscape; game-theoretic and cooperative-decision surveys~\cite{sun2025game,jin2025comprehensive} connect agent interaction to strategic and cooperative decision-making.

Our decomposition is also related to several theoretical traditions. In ensemble learning, the bias--variance--covariance decomposition~\cite{ueda1996generalization,wood2023unified}, the diversity-prediction theorem~\cite{hong2004groups,page2007difference}, and recent analyses of correlated failure in LLM ensembles~\cite{kim2025correlated} all emphasize that diversity alone is insufficient when errors are dependent. Separately, verifier-fidelity studies~\cite{lightman2023lets,setlur2024rewarding} analyze how model-based scoring can distort the quality signal used for selection. Our work adapts these intuitions to the test-time multi-agent reasoning setting.

Table~\ref{tab:related_work_comparison} summarizes representative methods through the lens of our decomposition. The goal is not to assign exclusive categories, but to highlight each method’s primary design emphasis and the dimensions that are less directly optimized.

\begin{table}[t]
\centering
\caption{Related methods through the gain decomposition lens. \ding{51}: primary emphasis; \ding{109}: secondary or implicit; \ding{55}: not a focus.}
\label{tab:related_work_comparison}
\resizebox{\columnwidth}{!}{%
\begin{tabular}{lcccp{3.5cm}}
\toprule
\textbf{Method} & $\mathcal{G}_{\text{explore}}$ & $\mathcal{G}_{\text{info}}$ & $\mathcal{G}_{\text{aggr}}$ & \textbf{Key Limitation} \\
\midrule
\multicolumn{5}{l}{\textit{Exploration-focused Methods}} \\
Self-Consistency~\cite{wang2023self} & \ding{51} & \ding{55} & \ding{109} & Correlated voting errors \\
Tree of Thoughts~\cite{yao2023tree} & \ding{51} & \ding{109} & \ding{109} & Heuristic, ungrounded evaluation \\
ReConcile~\cite{chen2024reconcile} & \ding{51} & \ding{109} & \ding{109} & Textual-only feedback \\
\midrule
\multicolumn{5}{l}{\textit{Information-focused Methods}} \\
ReAct~\cite{yao2023react} & \ding{109} & \ding{51} & \ding{55} & No parallel diversity \\
Reflexion~\cite{shinn2023reflexion} & \ding{109} & \ding{51} & \ding{55} & Sequential self-improvement only \\
LATS~\cite{zhou2024lats} & \ding{51} & \ding{51} & \ding{55} & Single-search aggregation only \\
\midrule
\multicolumn{5}{l}{\textit{Aggregation-focused Methods}} \\
Multi-Agent Debate~\cite{du2024improving} & \ding{109} & \ding{109} & \ding{51} & Ungrounded textual consensus \\
MoA~\cite{wang2024mixture} & \ding{109} & \ding{109} & \ding{51} & Weak external grounding \\
Two Heads~\cite{jin2025two} & \ding{109} & \ding{109} & \ding{109} & No explicit grounding \\
\midrule
\textbf{DIANOIA (Ours)} & \ding{51} & \ding{51} & \ding{51} & \textbf{Jointly targets all three dimensions} \\
\bottomrule
\end{tabular}%
}
\end{table}

In contrast, DIANOIA jointly enacts all three dimensions within a single workflow, addressing the fragmentation highlighted above.

\section{A Diagnostic Decomposition for Multi-Agent Reasoning}
\label{sec:theory}

Multi-agent systems have demonstrated remarkable improvements over single-agent baselines~\cite{wang2023self,du2024improving}, yet a fundamental question remains: \emph{where do these gains come from, and can we predict them?} We develop a \emph{diagnostic decomposition} that exposes three empirically measurable mechanisms---\textbf{Exploration}, \textbf{Information}, and \textbf{Aggregation}. The formalism is intentionally lightweight: the contribution lies not in mathematical novelty per se, but in a decomposition whose terms are individually measurable and prescriptive for system design---together yielding a four-rule diagnostic protocol (R1--R4, \S\ref{sec:method}) that forms qualitative bottleneck hypotheses from a task profile (verifier type, baseline $p$, answer-space topology). The claims below distinguish an exact identity, a conditional subadditive bound, and idealized bounds. R1--R4 are qualitative task-structure hypotheses, not calibrated consequences of the latter.

\begin{table}[t]
\centering
\caption{Connection between DIANOIA phases and diagnostic quantities.}
\label{tab:phase_quantity_map}
\setlength{\tabcolsep}{3pt}
\resizebox{\columnwidth}{!}{%
\begin{tabular}{lll}
\toprule
\textbf{Phase} & \textbf{Quantity} & \textbf{Mechanism} \\
\midrule
Propose & $C_K,\bar{\rho}$ & Candidate coverage/diversity \\
Execute/Review & $\eta^*(s),\hat{\epsilon}_0$ & Evidence fidelity \\
Synthesize & $\eta(f,s),r$ & Selection, repair, rescue \\
\bottomrule
\end{tabular}}
\end{table}

We formalize multi-agent reasoning as a tuple $(\mathcal{X}, \mathcal{T}, Q, K, \mathcal{E}, f)$ where $\mathcal{X}$ is the input space, $\mathcal{T}$ the solution space, $Q: \mathcal{T} \to \{0,1\}$ a quality indicator, $K$ proposers, $\mathcal{E}$ an executor providing feedback, and $f$ an aggregation function. We assume: (A1) \emph{baseline} conditional independence among proposers as the IID reference (relaxed via explicit pairwise correlation terms in Prop.~\ref{prop:diversity}; see App.~\ref{app:assumptions_discussion}), (A2) baseline success $p \in (0,1)$, (A3) finite strategy space (each agent's generation is bounded by maximum token length and finite vocabulary, ensuring $|\mathcal{T}| < \infty$), (A4) deterministic execution. Full definitions in Appendix~\ref{app:assumptions}.

We organize the decomposition below into three logically distinct layers---an exact identity, a definitional attribution, and a sub-additive upper bound---separated to make explicit which steps are tautological accounting and which are design-motivated modeling choices.

\begin{decomposition}[Identity, Attribution, and Bound]
\label{def:decomposition}
Under (A1)--(A4):

\noindent\textbf{(I) Exact identity} (definitional factorization; App.~\ref{app:gain_decomposition_proof}):
\begin{equation}\label{eq:mas_identity}
\mathbb{E}[Q(\tau^{\text{MAS}})] = \underbrace{C_K}_{\text{coverage}} \cdot \underbrace{\eta(f, s)}_{\text{effective efficiency}}
\end{equation}
where $C_K\!:=\!P(\bigvee_{k} Q(\tau^{(k)})\!=\!1)$ and $\eta(f,s)\!:=\!\mathbb{E}[Q(\tau^{\text{MAS}})]/C_K$ (Def.~\ref{def:aggr_efficiency}). $\eta$ equals the classical conditional accuracy $P(Q^{\text{MAS}}\!=\!1\mid E)$ for selector-only systems; for repair-capable DIANOIA it additionally absorbs synthesis's rescue contribution (quantify analysis in App.~\ref{app:synthesis_breakdown}).

\noindent\textbf{(II) Gain attribution} (definitional, each tied to a distinct design lever):
\textbf{(1)} $\mathcal{G}_{\text{explore}}\!:=\!C_K\!-\!p$; \textbf{(2)} $\mathcal{G}_{\text{info}}\!:=\!C_K[\eta^*(e)\!-\!\eta^*(\sigma)]$, where $\eta^*(s)\!:=\!\max_f \eta(f,s)$; \textbf{(3)} $\mathcal{G}_{\text{aggr}}\!:=\!C_K[\eta(f,s)\!-\!\eta(f_{\text{base}},s)]$, where $f_{\text{base}}$ is a task-appropriate non-iterative reference selector (e.g., majority vote or score-maximization; specified per analysis, see App.~\ref{app:gain_decomposition_proof}). \emph{This attribution is not unique}; we adopt it because each term corresponds to a distinct DIANOIA lever and admits an empirical estimator, with the synergy coefficient $\gamma$ in \S\ref{subsec:scaling_analysis} serving as a post-run descriptive subadditivity check; $\gamma$ is neither fitted nor used to construct R1--R4.

\noindent\textbf{(III) Conditional subadditive upper bound}. For selector-only systems, or for repair-capable systems satisfying the rescue-margin condition stated in Appendix~\ref{app:gain_decomposition_proof},
\begin{equation}\label{eq:decomposition}
\mathbb{E}[Q(\tau^{\text{MAS}})] - p \;\leq\; \mathcal{G}_{\text{explore}} + \mathcal{G}_{\text{info}} + \mathcal{G}_{\text{aggr}},
\end{equation}
strict in general due to multiplicative coupling (Remark~\ref{rem:subadditivity}; mechanism in App.~\ref{app:gain_decomposition_proof}, Part IV).
\end{decomposition}

\begin{remark}[Conceptual Orthogonality vs.\ Realized Subadditivity]
\label{rem:subadditivity}
Exploration, Information, and Aggregation are conceptually orthogonal (independently optimizable: \emph{generate} vs.\ \emph{evaluate} vs.\ \emph{combine}) but their realized gains are statistically coupled through the multiplicative structure $C_K\!\cdot\!\eta(f,s)$, yielding \emph{subadditivity} (actual gain $<$ sum), as in classical ensemble theory~\cite{krogh1995neural,wood2023unified} and collective intelligence~\cite{hong2004groups}. Empirically, the synergy coefficient $\gamma\!:=\!\Delta_{\text{joint}}/\sum_i\Delta_i$ falls in $[0.852,0.88]$ across MBPP and BFCL-SP (\S\ref{subsec:scaling_analysis}), consistent with the subadditive regime predicted by the multiplicative coupling.
\end{remark}

We characterize each dimension below; proofs are deferred to Appendix~\ref{app:proofs}.

\textbf{Exploration.} When structured diversity yields negative average pairwise success correlation $\bar{\rho}$, coverage increases relative to the IID reference:
\begin{proposition}[Exploration via Diversity]\label{prop:diversity}
The second-order Bonferroni coverage lower bound $\mathrm{LB}(\bar{\rho})$ satisfies $\mathrm{LB}(\bar{\rho})-\mathrm{LB}(0) = -\binom{K}{2}\bar{\rho}\,p(1{-}p) > 0$ for $\bar{\rho}\!<\!0$, i.e., the guaranteed coverage lower bound under diverse proposers strictly exceeds the IID baseline (App.~\ref{app:proofs}; strict comparison at actual-coverage level requires extra dependence assumptions).
\end{proposition}

\textbf{Information.} Verifier type sets an upper bound on achievable selection accuracy:
\begin{proposition}[Information Quality Bounds]\label{prop:information}
Under \emph{complete and deterministic verification} ($Q$ fully determined by $e$; Def.~\ref{def:exec_feedback}, A4), $I(Q;e)\!=\!H(Q)$. Textual feedback with non-zero error rates satisfies $I(Q;\sigma)\!<\!H(Q)$; pseudo-verification $\sigma_v$ satisfies $I(Q;\sigma)\!\leq\!I(Q;\sigma_v)\!\leq\!I(Q;e)$ (strict middle inequality empirical when $\sigma_v$ exploits specialization; App.~\ref{app:detailed_analysis}).
\end{proposition}
\begin{remark}[Verifier Regimes]\label{rem:task_classification}
$\mathcal{G}_{\text{info}}$ is set by verifier type: \emph{complete verification} (MBPP, BFCL-SP) gives $I(Q;e)\!=\!H(Q)$; \emph{pseudo-verification} (GSM8K, AIME) gives bounded $\mathcal{G}_{\text{info}}$; \emph{no external verifier} (many open-domain tasks) sharply limits $\mathcal{G}_{\text{info}}$, leaving only low-fidelity internal signals.
\end{remark}
By Fano-type bounds (Lemma~\ref{lem:fano}), higher-fidelity signals admit tighter achievable selection-error bounds; the realized $\mathcal{G}_{\text{info}}$ is therefore upper-bounded by verifier type, consistent with the cross-task pattern in \S\ref{subsec:discussion}.

\textbf{Aggregation.} Majority voting fails under correlated errors~\cite{condorcet1785,kim2025correlated}; evidence-based cross-review escapes this regime:
\begin{proposition}[Aggregation Efficiency; Idealized Bound]\label{prop:aggregation}
(a) \emph{Voting degradation under correlation.} As agent errors become highly correlated, majority voting loses its ensemble error-reduction advantage; in the worst case (adversarial correlation), the aggregation gain $\eta(f_{\mathrm{vote}})\!-\!\eta(f_{\mathrm{base}})$ vanishes. (b) \emph{DIANOIA's bound, under the simplified conditionally-independent reviewer model} (A6): with $K{-}1$ reviewers ($\epsilon_0\!<\!0.5$),
\[
\eta(f_{\mathrm{DIANOIA}})\!\geq\!1{-}\epsilon_0^{K-1}.
\]
This is a best-case bound: residual reviewer correlation $\rho_R\!>\!0$ floors the error at $\epsilon_\infty\!:=\!\epsilon_0\rho_R + \epsilon_0^{K-1}(1{-}\rho_R)$ in practice (\S\ref{subsec:behavior}).
\end{proposition}

Combining the three:
\begin{theorem}[DIANOIA Characterization]\label{thm:dianoia_convergence}
Under (A1)--(A6) (App.~\ref{app:assumptions}), DIANOIA satisfies (a) \emph{information sufficiency} $I(Q;e)\!=\!H(Q)$ on complete and deterministic verifiers; (b) \emph{finite-step convergence} as an exact potential game on complete and deterministic verifiers (approximate under pseudo-verification); (c) \emph{performance lower bound} $\mathbb{E}[Q]\!\geq\![1{-}(1{-}p)^K][1{-}\epsilon_0^{K-1+S}]$ on execution tasks, under the additional synthesis-iteration-independence idealization (A7, App.~\ref{app:proofs}). The exact/approximate/inapplicable regime distinction across verifier types is in App.~\ref{app:assumptions_discussion}.
\end{theorem}

\paragraph{Scope of formal claims.}
Equation~\ref{eq:mas_identity} is an exact definitional identity.
Equation~\ref{eq:decomposition} is a conditional subadditive bound for
selector-only systems, or for repair-capable systems satisfying the stated
rescue-margin condition. Proposition~\ref{prop:aggregation}(b) is an
idealized reviewer-independence bound under A6 and analyzes the
$R=K-1$ reviewer regime; DIANOIA's default configuration uses $R=1$.
Theorem~\ref{thm:dianoia_convergence}(b) applies exactly only under complete
and deterministic verification. The $S$-dependent tightening in
Theorem~\ref{thm:dianoia_convergence}(c) additionally relies on A7.
Under pseudo-verification these results provide qualitative guidance rather
than calibrated numerical guarantees.

\paragraph{Idealized theory versus measured practice.}
The idealized bounds predict exponential tightening with independent
reviewers and synthesis iterations. On MBPP, however, increasing synthesis
iterations from $S=1$ to $2$ to $3$ changes accuracy from
$83.6\%$ to $84.0\%$ to $84.6\%$, while increasing reviewers from $R=1$ to
$R=3$ changes accuracy from $84.6\%$ to $85.8\%$ at approximately
$2.6{\times}$ the token cost. The directions agree with the idealized
analysis, but the magnitudes do not: these bounds describe best-case trends,
not calibrated scaling laws.

\section{DIANOIA Methodology}
\label{sec:method}

DIANOIA enacts all three framework prescriptions in a four-phase workflow (Figure~\ref{fig:overview}): \emph{Propose} for Exploration, \emph{Execute} and \emph{Review} jointly for Information (raw evidence vs.\ its interpretation into actionable signals), and \emph{Synthesize} for Aggregation. No proper subset of these phases reproduces the gain (Section~\ref{subsec:scaling_analysis}). We detail each phase below.

\textbf{Phase 1: Propose ($\mathcal{G}_{\text{explore}}$).} DIANOIA assigns each proposer a distinct role---\textbf{Minimalist} (fewest steps), \textbf{Skeptic} (verify each step), \textbf{Explorer} (unconventional methods)---as one lightweight realization intended to reduce pairwise success correlation (Proposition~\ref{prop:diversity}). The governing prescription is reduced correlation rather than this particular role taxonomy. Full role prompts are in Appendix~\ref{app:exp_protocol}. Proposers generate candidates in parallel, with IID reference coverage $1\!-\!(1\!-\!p)^K$.

\textbf{Phase 2: Execute ($\mathcal{G}_{\text{info}}$).} Each candidate is evaluated via the highest-fidelity feedback mechanism available---sandboxed execution or LLM-based pseudo-verification---producing $e^{(k)}\!=\!\mathcal{E}(\tau^{(k)})$. By Theorem~\ref{thm:dianoia_convergence}a, complete and deterministic verification yields $I(Q;e)=H(Q)$.

\textbf{Phase 3: Review ($\mathcal{G}_{\text{info}}$).} $R$ reviewers per proposal perform evidence-based cross-review $v^{(k)}_j = \mathcal{R}_j(\tau^{(k)}, e^{(k)})$, transforming raw feedback into actionable quality signals (analyzing failure causes, identifying validated components, suggesting fixes). The default $R\!=\!1$ already captures the bulk of the aggregation gain (App.~\ref{app:potential_game_ideal}, Remark); the idealized peer-review regime $R\!=\!K{-}1$ underlies the $\epsilon_0^{K-1}$ misclassification bound of Proposition~\ref{prop:aggregation}b.

\begin{table*}[t]
\centering
\caption{Main results. Multi-agent methods (top) use Qwen3-30B-A3B; single-model references (bottom).}
\label{tab:main_results}
\setlength{\tabcolsep}{6pt}
\begin{tabular}{lcccc}
\toprule
\textbf{Method} & \textbf{GSM8K} & \textbf{AIME-2025} & \textbf{MBPP} & \textbf{BFCL-SP} \\
\midrule
Self-Consistency & 86.4\% \tiny[84.5, 88.3] & 56.7\% \tiny[40.0, 73.3] & 78.0\% \tiny[74.2, 81.2] & 82.3\% \tiny[78.5, 85.8] \\
Best-of-3 (same verifier) & 87.0\% \tiny[85.2, 88.8] & 73.3\% \tiny[56.7, 90.0] & 78.4\% \tiny[75.0, 81.8] & 84.0\% \tiny[80.5, 87.5] \\
MoA & 87.1\% \tiny[85.2, 88.9] & 86.7\% \tiny[73.3, 96.7] & 76.8\% \tiny[73.2, 80.2] & 85.8\% \tiny[82.5, 89.3] \\
Two Heads & 85.8\% \tiny[84.0, 87.6] & 80.0\% \tiny[66.7, 93.3] & 77.2\% \tiny[73.6, 81.2] & 88.8\% \tiny[85.5, 92.0] \\
ReConcile & 89.8\% \tiny[88.3, 91.5] & 70.0\% \tiny[53.3, 86.7] & 77.2\% \tiny[73.4, 81.0] & 82.3\% \tiny[78.3, 86.0] \\
\midrule
\textbf{DIANOIA (Ours)} & \textbf{91.1\%} \tiny[89.6, 92.7] & 93.3\%$^\ddagger$ \tiny[83.3, 100] & \textbf{84.6\%} \tiny[81.4, 87.0] & \textbf{92.3\%} \tiny[89.5, 94.8] \\
\midrule
Qwen3-30B-A3B & 83.6\% \tiny[81.5, 85.7] & 70.0\% \tiny[53.3, 86.7] & 76.0\% \tiny[72.4, 80.0] & 81.8\% \tiny[78.0, 85.8] \\
Qwen3-235B-A22B & 86.4\% \tiny[84.6, 88.3] & 73.3\% \tiny[56.7, 86.7] & 80.2\% \tiny[76.6, 83.8] & 89.0\% \tiny[85.8, 92.0] \\
DeepSeek-V3.2 & 85.3\% \tiny[83.3, 87.2] & 76.8\%$^\dagger$ \tiny[60.0, 90.0] & 81.2\% \tiny[77.6, 84.6] & 83.5\% \tiny[79.5, 87.3] \\
\bottomrule
\end{tabular}
\vspace{-1pt}

{\footnotesize $^\dagger$Extended reasoning (thinking mode) enabled. $^\ddagger$Complete 30-problem set; reported as a small-sample point estimate.}

\vspace{-8pt}
\end{table*}

\textbf{Review beyond Execute.} Execution gives \emph{whether} a candidate fails; Review adds \emph{why}, \emph{which sub-parts pass}, and \emph{how to repair}---feeding the synthesizer with the diagnostic content needed for repair (rather than mere selection).

\textbf{Phase 4: Synthesize ($\mathcal{G}_{\text{aggr}}$).} A synthesis agent integrates proposals, execution reports, and reviews through iterative refinement with closed-loop validation (Algorithm~\ref{alg:synthesis}, Appendix~\ref{app:exp_protocol}): trajectory grafting, re-execution, deterministic synthesis. Although the formal bound (Theorem~\ref{thm:dianoia_convergence}c) treats this phase as evidence-conditioned selection, in practice the synthesizer can perform local repair: instrumentation on MBPP shows $5.5\%$ of DIANOIA's correct outputs are \emph{rescued} cases where every initial proposal failed execution but the synthesizer recovered a correct answer using reviewer hints (Appendix~\ref{app:synthesis_breakdown}). The bound therefore lower-bounds, rather than fully characterizes, Phase~4.

\textbf{Diagnostic protocol.} Four rules from the framework map task profiles (verifier type, baseline $p$, answer-space topology) to qualitative bottleneck hypotheses for tasks with measurable quality signals. \textbf{R1}: a deterministic verifier lifts $\mathcal{G}_{\text{info}}$ to its ceiling, $I(Q;e){=}H(Q)$ (Thm.~\ref{thm:dianoia_convergence}a). \textbf{R2}: a high baseline $p$ leaves small exploration headroom $C_K{-}p$ (Prop.~\ref{prop:iid_exploration}). \textbf{R3}: a fragmented answer space renders majority voting unreliable (Failure Mode~1, App.~\ref{app:detailed_analysis}). \textbf{R4}: frequent per-instance $p_i{<}0.5$ makes Self-Consistency actively harm accuracy (Failure Mode~2). These hypotheses are evaluated empirically in \S\ref{subsec:discussion}.

\section{Experiments}
\label{sec:experiments}

We conduct comprehensive experiments to validate DIANOIA's effectiveness across four benchmarks spanning math, code, and tool use, and systematically ablate each gain dimension to demonstrate the necessity of joint optimization.

\subsection{Experimental Setup}
\label{subsec:setup}

\textbf{Datasets.} We evaluate on four benchmarks spanning math, code, and tool use, chosen for varied feedback regimes (Remark~\ref{rem:task_classification}): \emph{GSM8K} ($N{=}1319$, high $p$, pseudo-verifier)~\cite{cobbe2021training}; \emph{AIME-2025} ($N{=}30$, low $p$, pseudo-verifier)~\cite{aime2025}; \emph{MBPP} ($N{=}500$, deterministic exec)~\cite{austin2021program}; \emph{BFCL-SP} ($N{=}400$, schema-validation exec)~\cite{patil2025bfcl}. Full dataset details in Appendix~\ref{app:exp_protocol}.

\textbf{Baselines.} Self-Consistency~\cite{wang2023self} (sampling + majority vote), MoA~\cite{wang2024mixture} (layered aggregation), Two Heads~\cite{jin2025two} (collaborative reasoning), ReConcile~\cite{chen2024reconcile} (round-table discussion), and a \emph{Best-of-3} control sampling three IID candidates and selecting with the same execution signal as DIANOIA. Single-model references: Qwen3-30B-A3B~\cite{qwen3technicalreport} (base), Qwen3-235B-A22B, DeepSeek-V3.2~\cite{deepseekv3technicalreport}.

\textbf{Implementation.} All multi-agent methods use Qwen3-30B-A3B-Instruct-2507 with thinking mode disabled (zero-shot). DIANOIA default: $K{=}3, R{=}1, S{=}3$; baselines use the standard configurations from their respective papers. Full prompts, and other details in Appendix~\ref{app:exp_protocol}.

\subsection{Main Results}
\label{subsec:main_results}

Table~\ref{tab:main_results} presents results across all four benchmarks. Across GSM8K, MBPP, and BFCL-SP, DIANOIA outperforms every multi-agent baseline under the standard configurations.

DIANOIA improves over the strongest multi-agent baseline by $+$1.3pp (GSM8K), $+$6.6pp (MBPP), and $+$3.5pp (BFCL-SP); on MBPP its CI lower bound (81.4\%) exceeds every baseline's upper bound. The Best-of-3 control---which uses the same verifier as DIANOIA but without its multi-agent structure---reaches $78.4\%$ on MBPP, $6.2$pp below DIANOIA, isolating the marginal value of $\mathcal{G}_{\text{aggr}}$. AgentCoder~\cite{huang2024agentcoder}, a code-specific MAS benchmarked on its natural MBPP domain, reaches $81.2\%$ ($3.4$pp below DIANOIA on Qwen3-30B-A3B; cross-model results in \S\ref{subsec:cross_model}).

On the complete 30-problem AIME-2025 set, DIANOIA obtains a 93.3\% point estimate. An exploratory paired McNemar exact test against ReConcile yields $p=0.039$ from only nine discordant pairs; the wide confidence interval precludes a precise effect-size or model-superiority claim.

The cross-benchmark pattern matches the framework's predictions: DIANOIA's largest gains emerge on execution-grounded tasks where $I(Q;e)\!=\!H(Q)$ activates $\mathcal{G}_{\text{info}}$; Self-Consistency \emph{degrades} on AIME-2025 due to voting failure under low per-problem $p$; the GSM8K gain is bounded by pseudo-verifier fidelity. We formalize and verify these predictions in Section~\ref{subsec:discussion}, and use the budget-matched Pareto frontier (Section~\ref{subsec:efficiency_frontier}) as the primary fairness evidence for efficiency.

\subsection{Multi-Dimensional Scaling Analysis}
\label{subsec:scaling_analysis}

As a post-run descriptive check of subadditivity, we compute $\gamma=\Delta_{\mathrm{joint}}/\sum_i\Delta_i$ on the same MBPP and BFCL-SP runs used for the component analysis. The common empirical reference is $f_{\mathrm{base}}=(K,R,S)=(1,0,0)$. Because $\gamma$ is computed after the runs and is not held out, it is not used to define or validate R1--R4. We obtain $\gamma=0.88$ on MBPP and $\gamma=0.852$ on BFCL-SP; these are descriptive evidence of subadditivity rather than a validation of the idealized theory.

\begin{table}[t]
\centering
\caption{Joint-dimension scaling on MBPP.}
\label{tab:joint_optimization}
\setlength{\tabcolsep}{3pt}
\begin{tabular}{lccccl}
\toprule
\textbf{Config} & \textbf{K} & \textbf{R} & \textbf{S} & \textbf{Acc} & \textbf{Activates} \\
\midrule
Baseline & 1 & 0 & 0 & 76.0\% & --- \\
Explore-only & 3 & 0 & 1 & 81.2\% & $\mathcal{G}_{\text{explore}}$ \\
Info-only & 1 & 1 & 1 & 79.8\% & $\mathcal{G}_{\text{info}}$ \\
Aggr-only & 1 & 0 & 1 & 76.8\% & $\mathcal{G}_{\text{aggr}}$ \\
Two-dim & 3 & 1 & 1 & 83.6\% & Explore+Info \\
\textbf{DIANOIA-full} & \textbf{3} & \textbf{1} & \textbf{3} & \textbf{84.6\%} & \textbf{All three} \\
\bottomrule
\end{tabular}%
\end{table}

A \emph{role design ablation} (App.~\ref{app:scaling_tables}, Table~\ref{tab:role_ablation}) further isolates the exploration gain: \emph{any} structured role design (three-role triplet $84.6\%$, two-role subsets $81.4$--$84.4\%$, alternative triplet $83.4\%$) beats same-prompt$+$temperature diversity ($83.2\%$), but within the structured family, CIs overlap, so the gain comes from structured diversity as a \emph{principle}, not from any specific hand-crafted prompt. Single-dimension scaling sweeps and the BFCL-SP joint-optimization table ($\gamma\!=\!0.852$) are also in App.~\ref{app:scaling_tables}.

\textbf{Information-dimension diagnostics.} Reviewer error $\hat{\epsilon}_0\!=\!5.8\%$ on MBPP (exec) vs. $17.5\%$ on GSM8K (pseudo-verifier); the GSM8K pseudo-verifier attains precision $94.0\%$ / recall $86.6\%$. These numbers operationalize Remark~\ref{rem:task_classification} and feed the diagnostic analysis in Section~\ref{subsec:discussion}.

\subsection{Budget-Matched Pareto Frontier}
\label{subsec:efficiency_frontier}

Standard-configuration comparisons (Table~\ref{tab:main_results}) leave open whether each method could match DIANOIA at higher compute. We sweep each method along its native scaling knobs (39 configurations enumerated in App.~\ref{app:sweep_configs}) and plot the upper Pareto envelope of accuracy vs.\ tokens on MBPP (Figure~\ref{fig:efficiency_frontier}).

\begin{figure}[t]
    \centering
    \includegraphics[width=\columnwidth]{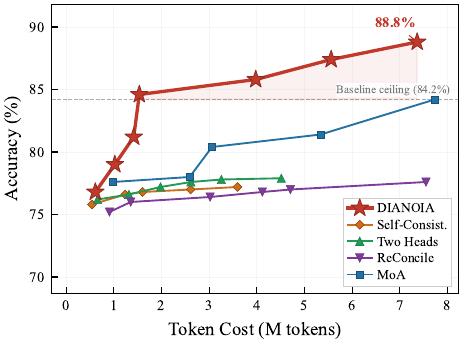}
    \caption{Token-budget Pareto frontier on MBPP.}
    \label{fig:efficiency_frontier}
    \vspace{-2pt}
\end{figure}

Self-Consistency, Two Heads, and ReConcile plateau below $78\%$ at every token budget we tested. MoA breaks through only by stacking three layers, reaching $84.2\%$ at $7.73$M tokens; DIANOIA reaches $84.6\%$ at just $1.54$M---a $\sim$$5{\times}$ token-budget advantage at matched accuracy---and continues to $88.8\%$ at $7.36$M tokens. Appendix~\ref{app:latency} reports tokens and wall-clock latency; the latter is serving-dependent.

\subsection{Cross-Model Robustness}
\label{subsec:cross_model}

\begin{table}
\centering
\caption{Cross-model robustness on Qwen3.6-35B-A3B~\cite{qwen2026qwen36}.}
\label{tab:cross_model}
\setlength{\tabcolsep}{2pt}
\resizebox{\columnwidth}{!}{%
\begin{tabular}{lccc}
\toprule
\textbf{Method} & \textbf{MBPP} & \textbf{BFCL-SP} & \textbf{GSM8K} \\
\midrule
\multicolumn{4}{l}{\textit{Qwen3-30B-A3B (headline base, Table~\ref{tab:main_results})}} \\
Single & 76.0 \tiny[72.4, 80.0] & 81.8 \tiny[78.0, 85.8] & 83.6 \tiny[81.5, 85.7] \\
Best-of-3 & 78.4 \tiny[75.0, 81.8] & 84.0 \tiny[80.5, 87.5] & 87.0 \tiny[85.2, 88.8] \\
AgentCoder$^\dagger$ & 81.2 \tiny[77.6, 84.4] & --- & --- \\
\textbf{DIANOIA} & \textbf{84.6} \tiny[81.4, 87.0] & \textbf{92.3} \tiny[89.5, 94.8] & \textbf{91.1} \tiny[89.6, 92.7] \\
\midrule
\multicolumn{4}{l}{\textit{Qwen3.6-35B-A3B (cross-version replication)}} \\
Single & 78.4 \tiny[74.6, 81.8] & 83.8 \tiny[80.3, 87.3] & 84.2 \tiny[82.3, 86.1] \\
Best-of-3 & 83.2 \tiny[79.8, 86.4] & 87.0 \tiny[83.8, 90.0] & 83.8 \tiny[81.8, 85.8] \\
AgentCoder$^\dagger$ & 88.4 \tiny[85.4, 90.8] & --- & --- \\
\textbf{DIANOIA} & \textbf{88.8} \tiny[86.0, 91.6] & \textbf{93.0} \tiny[90.5, 95.5] & \textbf{91.0} \tiny[89.5, 92.5] \\
\bottomrule
\end{tabular}%
}

{\footnotesize $^\dagger$Code-specific MAS; MBPP is its natural domain.}
\vspace{-2pt}
\end{table}

To check whether the gain pattern is Qwen3-30B-specific, we replicate the headline experiment on \textbf{Qwen3.6-35B-A3B} across MBPP, BFCL-SP, and GSM8K (Table~\ref{tab:cross_model})---spanning both deterministic and pseudo-verifier regimes. \textbf{DIANOIA leads on every cell at point-estimate level}; the gap over the single-model baseline \emph{persists or widens} on execution-grounded tasks ($+8.6\!\to\!+10.4$pp MBPP; $+10.5\!\to\!+9.2$pp BFCL-SP, within bootstrap CI overlap), indicating the prescriptions compose with rather than are absorbed by stronger base models.

The framework also predicts \emph{where Best-of-$N$ wins or loses}: on execution-grounded tasks BoN's gap over the single model widens with the stronger base ($+2.4\!\to\!+4.8$pp MBPP, $+2.2\!\to\!+3.2$pp BFCL-SP); but on pseudo-verified GSM8K, BoN's gap \emph{inverts} to $-0.4$pp---the lossy pseudo-verifier can no longer rank top-$N$ candidates reliably, injecting noise (P2: information-ceiling bottleneck). DIANOIA avoids this by using the verifier as input to richer aggregation rather than as the direct selection signal; the DIANOIA $-$ BoN margin therefore \emph{widens} on GSM8K ($+4.1\!\to\!+7.2$pp) even as it compresses on the execution-grounded tasks.

AgentCoder's MBPP gain widens with the stronger base ($+5.2\!\to\!+10.0$pp), showing iterative repair scales with single-pass quality. DIANOIA still leads ($+3.4$pp on Qwen3-30B-A3B; $+0.4$pp on Qwen3.6, within CI overlap); the margin contracts as both approach the MBPP execution ceiling, consistent with \S\ref{sec:limitations}'s operating regime.

\textbf{Cross-host evidence (Claude Code skill).} The same $K{=}3$, $S{=}3$, three-role prescriptions also realize as a Claude Code skill (App.~\ref{app:reproducibility}). With Claude Opus~4.7~\cite{anthropic2026claude}, the skill yields $+6.7$pp gain on BigCodeBench-Hard ($33.3\!\to\!40.0\%$, $N{=}30$) and Hard MBPP ($80.0\!\to\!86.7\%$, $N{=}15$)---a consistent direction on a different model family (Claude vs.\ Qwen3), though smaller $N$ precludes tight CIs.

\begin{table}[t]
\centering
\caption{Full-scale MBPP replication on additional model families
($N=500$). Values are accuracy with 95\% confidence intervals.}
\label{tab:cross_family_mbpp}
\setlength{\tabcolsep}{2pt}
\resizebox{\columnwidth}{!}{%
\begin{tabular}{lccc}
\toprule
\textbf{Model} & \textbf{Single} & \textbf{Best-of-3} & \textbf{DIANOIA} \\
\midrule
DeepSeek-V4-Flash
& 88.8 [86.2, 91.4]
& 97.0 [95.6, 98.3]
& \textbf{98.2 [97.0, 99.2]} \\
Gemini-3.5-Flash
& 93.0 [90.6, 95.0]
& 93.8 [91.6, 95.8]
& \textbf{96.8 [95.2, 98.2]} \\
\bottomrule
\end{tabular}}
\end{table}

Together with the full-scale Qwen results in Table~\ref{tab:cross_model},
these replications on DeepSeek-V4-Flash~\cite{deepseek2026v4} and
Gemini-3.5-Flash~\cite{google2026gemini35} show that DIANOIA consistently
improves over Single and achieves the highest point estimate among the
compared methods across Qwen, DeepSeek, and Gemini at full MBPP scale.

\subsection{Prospective Held-Out Diagnostic Check}
\label{subsec:vitabench}

The four original benchmarks provide within-study
prediction--verification consistency. VitaBench~\cite{he2025vitabench}
provides a prospective held-out diagnostic check. Before running DIANOIA, we
characterized VitaBench from its structural task profile: it has no
deterministic semantic verifier and requires long-horizon interactive
reasoning. The protocol therefore predicted that information fidelity would
be the primary bounded channel, while exploration and repair would retain
substantial headroom. This prediction was recorded without using measured
Single or DIANOIA accuracy.

\begin{table}[t]
\centering
\caption{Prospective held-out evaluation on VitaBench ($N=400$).}
\label{tab:vitabench}
\begin{tabular}{lcc}
\toprule
\textbf{Method} & \textbf{Pass rate} & \textbf{95\% CI} \\
\midrule
Single & 29.5\% & [25.2, 34.1] \\
\textbf{DIANOIA} & \textbf{45.0\%} & [40.2, 49.9] \\
\bottomrule
\end{tabular}
\end{table}

The resulting $+15.5$pp gain is consistent with the predicted regime.
VitaBench uses Qwen3.5-397B-A17B~\cite{qwen2026qwen35} because
Qwen3-30B-A3B produces a near-floor operating point on this substantially
harder benchmark. We therefore treat the base-model change as an explicit
scope caveat. Because evaluation relies on semantic pseudo-verification,
this experiment supplies empirical boundary evidence rather than an
oracle-level guarantee.

\subsection{Diagnostic Protocol: Verification}
\label{subsec:discussion}

Applying rules R1--R4 (\S\ref{sec:method}) to each benchmark yields four predictions: \textbf{P1} (R1$+$R3, exec-grounded MBPP/BFCL-SP): largest DIANOIA gains $+6.6/+3.5$pp; \textbf{P2} (R2, high-$p$ GSM8K): smallest gain $+1.3$pp; \textbf{P3} (R4, AIME-2025): SC degrades $-13.3$pp, the predicted voting-amplification regime; \textbf{P4} (BoN baseline): DIANOIA exceeds Best-of-$3$ by $+6.2$pp on MBPP, isolating $\mathcal{G}_{\text{aggr}}$ since BoN already activates Explore$+$Info. These are within-study consistency checks; R1--R4 are qualitative task-structure hypotheses, not calibrated theorem derivations. VitaBench supplies the prospective check: its pre-recorded hypothesis and subsequent 29.5\% to 45.0\% improvement are consistent. Per-prediction narratives are in Appendix~\ref{app:predictive_verification}.

\subsection{Behavioral Profile}
\label{subsec:behavior}

Four empirical operating numbers characterize DIANOIA: pairwise proposer correlation $\bar{\rho}\!=\!-0.15$ (drives coverage); reviewer error $\hat{\epsilon}_0\!=\!5.8\%$ on execution vs.\ $17.5\%$ on pseudo-verification (caps $\mathcal{G}_{\text{info}}$); GSM8K pseudo-verifier precision $94.0\%$/recall $86.6\%$ (caps math-task ceiling); synthesis breakdown $16.4/78.1/5.5\%$ across selection / selection-with-repair / rescue (App.~\ref{app:synthesis_breakdown}). Here $\bar{\rho}$ denotes proposer success correlation, not reviewer-error correlation, and therefore does not empirically validate A6. A controlled GSM8K stress test further shows that accuracy remains within 91--93\% through 25\% injected evidence corruption and falls to 88\% at 30\% (Table~\ref{tab:evidence_corruption} in the appendix).


\section{Conclusion}
\label{sec:conclusion}

Three measurable channels---coverage, information ceiling, and aggregation gap---organize observed cross-task variation in multi-agent reasoning gains, and the four-rule protocol (R1--R4, \S\ref{sec:method}) forms qualitative bottleneck hypotheses. The four original patterns and the prospective VitaBench check are consistent with these hypotheses. Full-scale MBPP replications show improvements over Single across Qwen, DeepSeek, and Gemini. On MBPP, DIANOIA matches the strongest baseline with approximately $5{\times}$ fewer tokens and dominates the measured budget--accuracy Pareto frontier.

Multi-agent design thus becomes \emph{channel-aware resource allocation}: practitioners can diagnose the likely bottleneck and invest tokens accordingly. Each channel admits a measurable estimator ($\bar{\rho}$, $\eta^*(s)$, aggregation gap; $\gamma$ as a post-run descriptive check), letting designers test components before committing. The framework provides a common axis for comparing future systems that instantiate these prescriptions differently; its exact identities, conditional bounds, and idealized guarantees have distinct scopes.

\section{Limitations}
\label{sec:limitations}

Our method improves performance across several reasoning settings, but its applicability and guarantees are bounded by two distinct types of constraint: conditions assumed by our theoretical analysis, and gaps in our empirical validation. We address each in turn.

\paragraph{Scope of tasks.}
The decomposition is most informative when a reasonably reliable quality signal is available. On tasks with deterministic verifiers---such as code execution or schema validation---the information channel $\mathcal{G}_{\text{info}}$ can operate at full capacity (Theorem~\ref{thm:dianoia_convergence}(a)). On tasks with model-based pseudo-verifiers---such as mathematical reasoning---it operates at reduced but measurable fidelity (precision $94.0\%$, recall $86.6\%$, \S\ref{subsec:behavior}). For open-ended generation, dialogue, or culturally grounded reasoning, reliable quality signals are often unavailable, so the practical value of $\mathcal{G}_{\text{info}}$ is substantially reduced and the framework becomes correspondingly less prescriptive. VitaBench provides initial boundary evidence for open-world, long-horizon interactive tasks, but it relies on semantic pseudo-verification rather than a deterministic oracle, and its use of Qwen3.5-397B-A17B prevents a strictly matched-base comparison with the core benchmarks. The result therefore supports the qualitative protocol at its boundary without extending exact verifier-dependent guarantees to open-ended tasks. Extending the decomposition to weakly verifiable or verifier-free settings is an important open direction.

\paragraph{Effective operating regime.}
DIANOIA is most effective when exploration headroom and verifier signal are jointly substantial, i.e., in a broad mid-accuracy regime rather than at either extreme. Two boundary conditions limit generality. \emph{High baseline}: residual headroom is small and the additional token cost may not be justified. \emph{Very low baseline}: role-diverse proposers may fail to surface any correct candidate, shifting more of the burden from cross-review to the synthesizer's repair capability. Table~\ref{tab:cross_model} provides one empirical calibration point; practitioners should locate their task on this curve before deployment.

\paragraph{Latency and compute proxy.}
Throughout the paper we report token count as the primary compute proxy, which captures model usage but not wall-clock latency or infrastructure cost under API rate limits. DIANOIA's main sequential bottleneck is the $S$ synthesis iterations; Phases 1 and 3 are already parallelized over $K$ and $K\!\cdot\!R$ agents respectively ($R\!=\!1$ in the default configuration; see App.~\ref{app:complexity}). The latency profile is approximately
\[
  L_{\text{propose}} + L_{\text{execute}} + L_{\text{review}}
  + S\cdot(L_{\text{synth}} + L_{\text{execute}}).
\]
Measured wall-clock latency is reported in Appendix~\ref{app:latency}; DIANOIA is substantially slower than Single because proposal/review and iterative synthesis introduce additional calls. These measurements depend on provider load, API scheduling, and available parallelism, so the cost--accuracy tradeoffs may differ in production settings and do not constitute a hardware-independent latency bound.

\paragraph{Empirical scope.}
AIME-2025 ($N{=}30$) supports the gain direction but not a precise estimate of effect size; its wide confidence interval and nine discordant pairs do not support a model-superiority claim. We therefore treat the complete 30-problem set as a small-sample stress test, and broader competition-math evaluation is needed before drawing quantitative conclusions about the low-baseline regime.

\paragraph{Theoretical idealizations.}
Two formal results assume conditions that hold exactly on execution-grounded tasks but only approximately on pseudo-verified ones. Proposition~\ref{prop:aggregation}(b)'s exponential reviewer-error bound assumes conditional independence among reviewers~(A6); agents instantiated from the same base model can retain correlated errors, so the bound should be interpreted as a best-case trend rather than a tight numerical guarantee. The reported $\bar{\rho}\!=\!-0.15$ measures correlation among proposer success events and supports the intended diversity regime; it does not measure reviewer-error independence or validate A6. Proposition~\ref{prop:aggregation}(b) also analyzes $R\!=\!K{-}1$ reviewers, whereas the default configuration uses $R\!=\!1$. Theorem~\ref{thm:dianoia_convergence}(a)'s information-sufficiency result requires a complete-and-deterministic verifier (Definition~\ref{def:exec_feedback}); GSM8K, AIME-2025, and VitaBench instead use pseudo-verification and therefore support empirical diagnostics rather than exact information sufficiency. The $S$-dependent tightening additionally relies on the A7 synthesis-iteration-independence idealization, although repeated failures can be correlated in practice.

\bibliography{sample-base}

\appendix

\section{Reproducibility Statement}
\label{app:reproducibility}
We design the paper to be reproducible from the artifacts described below.

\textbf{Code and data.} The complete implementation---including all four DIANOIA phases, experimental drivers, diagnostic-metric scripts, and the Claude Code skill---is publicly available at \url{https://github.com/sachiel321/DIANOIA4MALLM}.

\textbf{Quick experiential evaluation.} The released repository additionally bundles a self-contained Claude Code skill (\texttt{skill/}, installed via \texttt{bash skill/install.sh}, invoked as \texttt{/dianoia <task>}) that lets users apply DIANOIA to users' own everyday tasks from within any Claude Code session---no LLM client setup or API key required. The Python API above remains the reproducibility form factor for paper numbers; measured skill gains on Claude Opus~4.7 (\S\ref{subsec:cross_model}) provide complementary cross-host evidence.

\textbf{Models and inference.} The base model for the core multi-agent comparisons is Qwen3-30B-A3B-Instruct-2507 with thinking mode disabled; reference models are Qwen3-235B-A22B-Instruct-2507 and DeepSeek-V3.2 (instruct mode, with thinking enabled only for AIME-2025 as documented in Appendix~\ref{app:exp_protocol}). Cross-version experiments use Qwen3.6-35B-A3B~\cite{qwen2026qwen36}; VitaBench alone uses Qwen3.5-397B-A17B~\cite{qwen2026qwen35}; and the full MBPP family replications use model identifiers \texttt{deepseek-v4-flash}~\cite{deepseek2026v4} and Gemini-3.5-Flash~\cite{google2026gemini35}. All full-scale MBPP replications use the same 500-example split, prompts, evaluator, and Single/Best-of-3/DIANOIA method protocol. DIANOIA uses $K{=}3$, $R{=}1$, and $S{=}3$. Sampling uses temperature $0.7$ and \texttt{seed=42} where supported for stochastic methods, and temperature $0$ for deterministic synthesis; unsupported seeds are recorded as such by the released drivers.

\textbf{Prompt templates.} The exact role prompts (Minimalist, Skeptic, Explorer), reviewer template, synthesizer template, and pseudo-verifier template are reproduced verbatim in Appendix~\ref{app:exp_protocol}. The prompts are short enough to be transcribed without ambiguity.

\textbf{Statistical reporting.} Accuracy intervals use bootstrap resampling ($1{,}000$ iterations); the implementation is provided in the released code (\texttt{bootstrap\_ci}). VitaBench pass-rate intervals are 95\% Wilson score intervals. The AIME-2025 paired test (DIANOIA vs.\ ReConcile, $p{=}0.039$) is computed by the standard McNemar exact test on the $2{\times}2$ table of paired per-problem outcomes; a reference implementation is provided as the \texttt{analyze\_rebuttal\_metrics mcnemar} entry point in the released code, which consumes any user-supplied pair of per-problem binary outcome vectors. Ablations and joint-dimension analyses are reported on the full benchmark splits.

\textbf{VitaBench chronology.} Before either Single or DIANOIA accuracy was measured, we recorded from the task structure---absence of a deterministic semantic verifier and long-horizon interactive reasoning---the hypothesis that information fidelity would be the primary bounded channel while exploration and repair retained headroom. We then ran the full held-out evaluation set ($N{=}400$) with semantic pseudo-verification and the benchmark pass-rate protocol. Qwen3.5-397B-A17B replaces the core Qwen3-30B-A3B only here because the latter produces a near-floor operating point; this base-model change and the non-oracle verifier limit the result to prospective empirical boundary evidence.

\textbf{Reviewer-error estimation.} The judgment unit is one reviewer--proposal pair. For each pair, we compare the binary reviewer judgment with ground-truth candidate quality, conditional on the evidence shown to that reviewer, and report the resulting aggregate per-judgment mismatch rate $\hat{\epsilon}_0$. This yields 5.8\% on MBPP and 17.5\% on GSM8K.

\textbf{Licenses and intended use.} All released artifacts (code, adapters, diagnostic metric scripts, and Claude Code skill) are distributed under the MIT License. The benchmarks we use are publicly available under their respective licenses (GSM8K, AIME-2025, MBPP, BFCL-SP, and VitaBench) and are used for their intended research-evaluation purpose. Base LLMs are accessed via standard cloud inference APIs under their providers' terms of service.

\textbf{Compute infrastructure.} All inference is conducted through cloud-based LLM inference APIs; no GPU training is performed. Token consumption is reported throughout as the primary compute proxy (see \S\ref{sec:limitations} for caveats about wall-clock latency).

\textbf{Use of AI assistants.} AI assistants were used during paper preparation for LaTeX editing, citation formatting, and code refinement of the released implementation. All conceptual contributions---the gain decomposition, the four-rule diagnostic protocol, the DIANOIA system design, theorems, proofs, and experimental designs---are the authors' own. AI-generated code was reviewed and tested before inclusion.

\section{Evidence-Corruption Robustness}
\label{app:evidence_corruption}

\begin{table}[t]
\centering
\caption{Controlled evidence-corruption stress test on the first 100 GSM8K
examples.}
\label{tab:evidence_corruption}
\resizebox{\columnwidth}{!}{%
\begin{tabular}{lrrrrrrr}
\toprule
Flip rate (\%) & 0 & 5 & 10 & 15 & 20 & 25 & 30 \\
\midrule
Accuracy (\%) & 92 & 93 & 92 & 91 & 93 & 91 & 88 \\
\bottomrule
\end{tabular}}
\end{table}

Accuracy remains within 91--93\% through 25\% corruption and falls to 88\%
at 30\%. This intervention corrupts pseudo-verifier evidence before
Review/Synthesize; it is therefore an evidence-robustness test, not a direct
sweep of the conditional reviewer-error parameter $\epsilon_0$.

\section{Token and Latency Measurements}
\label{app:latency}

\begin{table}[t]
\centering
\caption{Mean input/output tokens and wall-clock latency per task under
standard configurations with Qwen3.6-35B-A3B.}
\resizebox{\columnwidth}{!}{%
\begin{tabular}{lcc}
\toprule
\textbf{Benchmark} & \textbf{Single: in/out; sec} &
\textbf{DIANOIA: in/out; sec} \\
\midrule
MBPP    & 232/145; 1.50  & 2,495/1,211; 12.02 \\
BFCL-SP & 693/34; 0.91   & 2,628/979; 9.93 \\
GSM8K   & 71/368; 2.67   & 2,431/1,455; 14.87 \\
\bottomrule
\end{tabular}}
\end{table}

\section{Formal Definitions}
\label{app:definitions}

This section provides complete formal definitions referenced in Section~\ref{sec:theory}.

\begin{definition}[Reasoning Task]
\label{def:task}
A reasoning task is defined by the tuple $(\mathcal{X}, \mathcal{T}, Q)$, where:
\begin{itemize}[leftmargin=*,topsep=0pt,itemsep=1pt]
    \item $\mathcal{X}$ is the input space (e.g., natural language problem descriptions)
    \item $\mathcal{T}$ is the output space (e.g., solution trajectories or code)
    \item $Q: \mathcal{T} \times \mathcal{X} \to \{0, 1\}$ is a binary quality function where $Q(\tau, x) = 1$ if and only if $\tau$ is a correct solution for input $x$
\end{itemize}
\textbf{Notation convention.} For a fixed input $x$ (every probability statement in the paper conditions on a sampled or fixed input from $\mathcal{X}$), we write $Q(\tau) := Q(\tau, x)$ as shorthand throughout the main text and proofs. All expectations $\mathbb{E}[Q(\cdot)]$ are understood as averages over both the input distribution and any stochasticity in the multi-agent procedure.
\end{definition}

\begin{definition}[Multi-Agent System]
\label{def:mas}
A multi-agent reasoning system is characterized by $(K, \mathcal{P}, \mathcal{E}, f)$:
\begin{itemize}[leftmargin=*,topsep=0pt,itemsep=1pt]
    \item $K \in \mathbb{N}^+$: number of proposing agents
    \item $\mathcal{P} = \{P_1, \ldots, P_K\}$: generative distributions for each proposer
    \item $\mathcal{E}: \mathcal{T} \to \mathcal{R}$: environment executor mapping solutions to structured feedback
    \item $f: \mathcal{T}^K \times \mathcal{R}^K \to \mathcal{T}$: aggregation function synthesizing the final solution
\end{itemize}
\end{definition}

\begin{definition}[Success Correlation]
\label{def:correlation}
For agents $i, j$, define the success correlation coefficient:
\begin{align}
\rho_{ij} &:= \text{Corr}(Q(\tau^{(i)}), Q(\tau^{(j)})) \nonumber\\
&= \tfrac{\mathbb{E}[Q(\tau^{(i)}) Q(\tau^{(j)})] - p^2}{p(1-p)}
\end{align}
where $p = P(Q(\tau^{(k)}) = 1)$ is the marginal success probability.
\end{definition}

\begin{definition}[Mutual Information]
\label{def:mutual_info}
The mutual information between quality $Q$ and signal $S$ is:
\begin{align}
I(Q; S) &:= H(Q) - H(Q \mid S) \nonumber\\
&= \textstyle\sum_{q, s} P(q, s) \log \tfrac{P(q, s)}{P(q)P(s)}
\end{align}
where $H(Q) = -\sum_q P(q) \log P(q)$ is the entropy of $Q$. Note that $I(Q; S) \leq H(Q)$ with equality if and only if $S$ completely determines $Q$.
\end{definition}

\begin{definition}[Execution Feedback]
\label{def:exec_feedback}
Execution feedback $e = \mathcal{E}(\tau)$ is a structured report containing:
\begin{itemize}[leftmargin=*,topsep=2pt,itemsep=1pt]
    \item $e.\mathtt{success} \in \{0, 1\}$: whether execution completed without errors
    \item $e.\mathtt{tests} \in \{0, 1\}^m$: pass/fail status for $m$ test cases
    \item $e.\mathtt{error} \in \Sigma^*$: error messages and stack traces (if any)
\end{itemize}
For code generation, quality is typically: $Q(\tau) := \mathbf{1}[e.\mathtt{success} \wedge e.\mathtt{tests} = \mathbf{1}]$.
\end{definition}

\begin{definition}[Textual Feedback]
\label{def:text_feedback}
Textual feedback $\sigma$ is an LLM-generated evaluation of solution quality, represented as a categorical variable (e.g., $\sigma \in \{\text{``correct''}, \text{``incorrect''}, \text{``uncertain''}\}$). For the binary case, define two directional error parameters:
\begin{align}
\alpha &:= P(\sigma\!=\!\text{``c''} \mid Q\!=\!0), \\
\beta &:= P(\sigma\!=\!\text{``i''} \mid Q\!=\!1),
\end{align}
where ``c'' and ``i'' abbreviate ``correct'' and ``incorrect.''
\end{definition}

\begin{definition}[Model-Based Pseudo-Verification Feedback]
\label{def:pseudo_verification}
For tasks without deterministic execution environments (e.g., mathematical reasoning), model-based pseudo-verification feedback $\sigma_v = \mathcal{V}(\tau, x)$ is a structured diagnostic report produced by a dedicated LLM evaluator that analyzes solution $\tau$ for input $x$ \textit{without access to the ground truth}. The report contains:
\begin{itemize}[leftmargin=*,topsep=2pt,itemsep=1pt]
    \item $\sigma_v.\mathtt{is\_correct} \in \{0, 1\}$: the verifier's binary quality judgment
    \item $\sigma_v.\mathtt{confidence} \in [0, 1]$: a calibrated confidence score
    \item $\sigma_v.\mathtt{errors} \in \Sigma^*$: identified reasoning errors and diagnostic feedback
\end{itemize}
Since $\sigma_v$ is generated by an LLM evaluator (rather than deterministic execution), it satisfies $I(Q; \sigma) \leq I(Q; \sigma_v) \leq I(Q; e)$: weakly more informative than unstructured self-critique $\sigma$ (due to specialized evaluation prompts and structured output)---with strict middle inequality expected empirically when $\sigma_v$ exploits its specialization---and weakly less informative than deterministic execution feedback $e$ (due to the verifier's own error rates).
\end{definition}

\begin{definition}[Aggregation Efficiency / Effective System Efficiency]
\label{def:aggr_efficiency}
The (effective) efficiency of aggregation function $f: \mathcal{T}^K \times \mathcal{R}^K \to \mathcal{T}$ is the ratio of expected output quality to expected coverage:
\begin{equation}
\eta(f) := \frac{\mathbb{E}[Q(f(\boldsymbol{\tau}, \boldsymbol{e}))]}{\mathbb{E}[\max_k Q(\tau^{(k)})]} = \frac{\mathbb{E}[Q(f(\boldsymbol{\tau}, \boldsymbol{e}))]}{C_K},
\end{equation}
where the second equality uses $\mathbb{E}[\max_k Q(\tau^{(k)})] = P(\exists k\!:\!Q(\tau^{(k)})\!=\!1) = C_K$ for binary $Q$. For \emph{selector-only systems} (final output is one of the original proposals), $\eta(f) \leq 1$ with $\eta(f) = 1$ indicating perfect oracle-quality selection; in this regime $\eta(f) = P(Q^{\text{MAS}}\!=\!1 \mid \exists k\!:\!Q^{(k)}\!=\!1)$, the classical conditional selection accuracy. For \emph{repair-capable systems} (synthesizer can produce outputs absent from the proposal set), $\eta(f)$ additionally absorbs any synthesis-rescue contribution and represents the broader effective success rate per unit coverage. When the dependence on the feedback signal $s$ matters, we write $\eta(f, s)$ as in Decomposition~\ref{def:decomposition}; $\eta(f)$ is used when the signal is fixed or implicit.
\end{definition}

\section{Formal Assumptions}
\label{app:assumptions}

We explicitly state all assumptions underlying our theoretical analysis. The first four assumptions (A1-A4) formalize standard properties of LLM-based multi-agent systems. We add two additional assumptions (A5-A6) specific to aggregation analysis.

\begin{assumption}[Baseline Conditional Independence]
\label{assump:independence}
Given input $x$, each proposer generates output independently as the IID reference:
\begin{equation}
P(\tau^{(1)}, \ldots, \tau^{(K)} \mid x) = \prod_{k=1}^{K} P_k(\tau^{(k)} \mid x).
\end{equation}
Deviations from independence---in particular, the negative success correlation induced by role specialization---are modelled explicitly via the pairwise correlation term $\bar{\rho}$ in Proposition~\ref{prop:diversity} and discussed in App.~\ref{app:assumptions_discussion}.
\end{assumption}

\textbf{Justification.} Proposers operate in parallel without communication during the generation phase, conditioned only on the input $x$ and their assigned roles. The IID form serves as an analytical \emph{reference baseline} for coverage analysis in Section~\ref{sec:theory}; we do not claim that LLM agents are truly independent, and the empirical $\bar{\rho}\!\approx\!-0.15$ we measure (App.~\ref{app:assumptions_discussion}) confirms that the real regime departs from A1 in a direction that strengthens, rather than weakens, the resulting coverage.

\begin{assumption}[Baseline Competence]
\label{assump:baseline}
There exists baseline success probability $p \in (0, 1)$ such that $P(Q(\tau^{(k)}) = 1) = p$ for all $k$ in the absence of role specialization. Under role specialization, the per-role success rates $p_k := P(Q(\tau^{(k)}) = 1)$ need not all equal $p$; we write $\bar p := \frac{1}{K}\sum_k p_k$ for the averaged marginal.
\end{assumption}

\textbf{Justification.} This assumes agents have non-trivial but imperfect capability ($0 < p < 1$), matching empirical LLM performance on reasoning tasks. Role specialization redistributes (rather than degrades) marginal accuracy: roles are designed for orthogonal failure modes, not for raw accuracy gains. We make no assumption about the magnitude or distribution of the per-role spread $\{p_k - \bar p\}$.

\emph{Convention for closed-form expressions: a three-level rigorous ladder.} Subsequent propositions express coverage and pairwise-intersection terms through a single symbol $p$ for tractability. The closed forms admit progressively tighter rigorous interpretations as we layer in structural assumptions:

\medskip
\textbf{Level 1 — Heterogeneous IID (under A1 alone): AM--GM lower bound.}
Under conditional independence (A1) with arbitrary $\{p_k\}$ of mean $\bar p$, the true coverage is $C_K^{\text{true}} = 1 - \prod_k(1{-}p_k)$. By AM--GM, $\prod_k (1{-}p_k) \leq (1{-}\bar p)^K$, hence
\begin{equation}\label{eq:level1_amgm}
C_K^{\text{true}} \;\geq\; \underbrace{1 - (1{-}\bar p)^K}_{=:\,C_K^{\text{unif}}}.
\end{equation}
The closed-form $C_K^{\text{unif}}$ is therefore a guaranteed lower bound for \emph{any} $\{p_k\}$ with mean $\bar p$. Equivalently, $C_K^{\text{true}}$ is Schur-convex in $(p_1,\ldots,p_K)$, so the uniform vector minimizes coverage subject to fixed mean. The strict positivity and monotonicity of $\mathcal{G}_{\text{explore}}\!=\!C_K^{\text{unif}}\!-\!\bar p$ thus carry over: under heterogeneity the true exploration gain is at least as large as the closed-form value.

\medskip
\textbf{Level 2 — Diversity-induced negative dependence (A1$'$): tighter IID-heterogeneous lower bound on \emph{actual coverage}.}
The negative pairwise correlation $\bar\rho<0$ we measure empirically (Section~\ref{subsec:behavior}) is consistent with a stronger structural property: that role-specialized success indicators are \emph{negatively associated} \citep{joagdev1983na} in the sense of the orthant inequality
\begin{equation*}\label{eq:na_orthant}
\textstyle P\!\bigl(\bigcap_k A_k^c \mid x\bigr) \;\leq\; \prod_k (1-p_k), \tag{A1$'$}
\end{equation*}
i.e., the joint probability of \emph{all} agents failing is no larger than under independence. This is the canonical multivariate formalization of ``mutual inhibition'': when role $i$'s specialty matches the problem (raising $P(A_i)$), role $j$ with an orthogonal specialty is correspondingly less likely to succeed, so the joint failure event is suppressed below the independence baseline. (A1$'$) is weaker than full negative association of \citet{joagdev1983na} but sufficient for our coverage bound and verifiable via the joint outcome statistics; it is implied by A1 (with equality) and strictly strengthens any pairwise-correlation-only condition.

Under (A1$'$) and arbitrary $\{p_k\}$ with mean $\bar p$, the coverage satisfies the two-step rigorous chain
\begin{align}
C_K^{\text{true}} \;=\; 1-\textstyle P(\bigcap_k A_k^c)
&\;\overset{\text{(A1$'$)}}{\geq}\; 1-\textstyle\prod_k(1{-}p_k) \nonumber\\
&\;\overset{\text{AM--GM}}{\geq}\; 1-(1{-}\bar p)^K.\label{eq:level2_na}
\end{align}
The first inequality is a bound on \emph{actual coverage} (not an LB-on-LB), strictly tighter than Level~1: it quantifies the diversity bonus $\Delta_{\text{div}} := \prod_k(1{-}p_k) - P(\bigcap_k A_k^c) \geq 0$ as the gap between independence and negatively-dependent joint failure. The middle floor $1-\prod_k(1{-}p_k)$ also strictly exceeds the uniform AM--GM floor whenever the $\{p_k\}$ are non-constant (the AM--GM inequality is strict unless all $p_k$ are equal), so Level~2 strengthens Level~1 by two distinct mechanisms: (i) joint negative dependence among role-specialized indicators (the structural gain certified by (A1$'$)), and (ii) Schur convexity of $1-\prod_k(1{-}p_k)$ in $\{p_k\}$, which exploits any heterogeneity in per-role marginals. Both gains are quantitative, not cosmetic, and neither requires any empirical assumption about the magnitude of the heterogeneity.

\medskip
\textbf{Level 3 — Sharp closed form (K{=}2 or uniform).}
For $K=2$, inclusion--exclusion terminates, and the diversity bonus admits a sharp closed form
\begin{align}\label{eq:level3_sharp}
\Delta_{\text{div}}^{K=2} &\;=\; -\rho_{12}\sqrt{p_1(1{-}p_1)\,p_2(1{-}p_2)} \nonumber\\
&\;\geq\; 0 \quad \text{when } \rho_{12}\leq 0,
\end{align}
reducing to $-\rho_{12}\,\bar p(1{-}\bar p)$ when $p_1=p_2=\bar p$. This is the form appearing in Proposition~\ref{prop:diversity}'s Bonferroni LB for general $K$. For $K\!\geq\!3$ the bonus does not admit a closed form in $(\bar p, \bar\rho)$ alone (it depends on higher-order joint structure), but Level~2's rigorous chain \eqref{eq:level2_na} certifies its non-negativity and gives a strictly tighter quantitative floor than Level~1's $1-(1-\bar p)^K$. Proposition~\ref{prop:diversity}'s Bonferroni LB-vs-LB comparison should be read as a closed-form qualitative summary of this Level~3 effect, with Level~2 supplying the rigorous backbone under (A1$'$).

\medskip
\textbf{Summary.} Level~1 transfers to heterogeneous $\{p_k\}$ with no extra assumption (AM--GM). Level~2 strictly tightens this under the structural assumption (A1$'$), which formalizes ``role-specialization $\Rightarrow$ mutual inhibition'' as a joint-distribution property and is consistent with all four empirically reported negative pairwise correlations. Level~3 provides the sharp closed form in the $K\!=\!2$ or uniform regime. None of the three levels depends on any specific bound on $\max_k |p_k - \bar p|$.

\begin{assumption}[Finite Strategy Space]
\label{assump:finite}
The output space has finite cardinality: $|\mathcal{T}| < \infty$.
\end{assumption}

\textbf{Justification.} While the space of all possible text sequences is technically infinite, practical constraints (maximum token length, finite vocabulary) render it finite. This assumption is essential for potential game convergence (Theorem~\ref{thm:dianoia_convergence}b).

\begin{assumption}[Execution Determinism]
\label{assump:determinism}
For executable tasks, the environment executor is deterministic: given solution $\tau$, the feedback $e = \mathcal{E}(\tau)$ is uniquely determined.
\end{assumption}

\textbf{Justification.} Code execution, test evaluation, and MCP calls produce deterministic outputs (modulo explicit randomness in the code itself). This assumption directly applies to tasks with environmental verifiers (MBPP, BFCL-SP) and enables Theorem~\ref{thm:dianoia_convergence}a (execution feedback as sufficient statistic). For mathematical reasoning tasks (GSM8K, AIME) where no deterministic executor exists, DIANOIA substitutes model-based pseudo-verification $\sigma_v$ (Definition~\ref{def:pseudo_verification}); A4 does not hold in this regime, and theoretical guarantees apply only approximately (see Appendix~\ref{app:assumptions_discussion}).

\begin{assumption}[Reviewer Accuracy]
\label{assump:reviewer_accuracy}
Each reviewer, when provided with objective evidence (execution feedback $e$ or structured pseudo-verification feedback $\sigma_v$), correctly assesses solution quality with probability $1 - \epsilon_0$, where $\epsilon_0 \in (0, 0.5)$.
\end{assumption}

\textbf{Justification.} LLMs are imperfect interpreters of feedback signals and can misclassify solution quality. The constraint $\epsilon_0 < 0.5$ ensures reviewers are better than random guessing.

\begin{assumption}[Reviewer Independence]
\label{assump:reviewer_independence}
Given the true quality $Q(\tau^{(k)})$ and the available evidence (execution feedback $e^{(k)}$ or pseudo-verification feedback $\sigma_v^{(k)}$), different reviewers' assessments are conditionally independent.
\end{assumption}

\textbf{Justification.} Reviewers operate independently, each analyzing the same evidence. While they may share systematic biases (violating full independence), conditional on the objective evidence, remaining errors are largely uncorrelated. This assumption enables the exponential error reduction in Proposition~\ref{prop:aggregation}b.

\section{Detailed Analysis and Examples}
\label{app:detailed_analysis}

This section provides numerical examples, quantitative comparisons, and detailed case studies supporting the main theoretical results.

\subsection{Exploration Gain: Quantitative Analysis}

\begin{proposition}[IID Exploration Gain]
\label{prop:iid_exploration}
With $K$ independent agents each having success probability $p$, the exploration gain is:
\begin{equation}
\mathcal{G}_{\text{explore}}^{\text{iid}}(K) = 1 - (1-p)^K - p
\end{equation}
This gain is strictly positive for $K \geq 2$ and monotonically increasing in $K$.
\end{proposition}

\textbf{Numerical Illustration.} For baseline success rate $p = 0.4$:
\begin{itemize}[leftmargin=*,topsep=2pt,itemsep=1pt]
    \item $K=2$: $\mathcal{G}_{\text{explore}} = 0.64 - 0.40 = 0.24$ (24\% improvement)
    \item $K=3$: $\mathcal{G}_{\text{explore}} = 0.784 - 0.40 = 0.384$ (38.4\% improvement)
    \item $K=5$: $\mathcal{G}_{\text{explore}} = 0.922 - 0.40 = 0.522$ (52.2\% improvement)
    \item $K=10$: $\mathcal{G}_{\text{explore}} = 0.994 - 0.40 = 0.594$ (59.4\% improvement)
\end{itemize}
The marginal gain $\mathcal{G}_{\text{explore}}(K+1) - \mathcal{G}_{\text{explore}}(K)$ decreases as $K$ grows, exhibiting diminishing returns characteristic of parallel sampling.

\textbf{Role Diversity Enhancement.} Consider three specialized roles:
\begin{itemize}[leftmargin=*,topsep=2pt,itemsep=1pt]
    \item \textbf{Minimalist}: Prefers short, direct solutions. Fails on edge cases requiring extensive validation (failure mode: insufficient coverage).
    \item \textbf{Skeptic}: Adds redundant checks. Fails on time/resource-constrained problems (failure mode: over-engineering).
    \item \textbf{Explorer}: Seeks non-standard approaches. Fails on tasks with strict conventions (failure mode: excessive creativity).
\end{itemize}
These roles exhibit \textit{orthogonal failure modes}. Empirically, on MBPP dataset, we observe pairwise correlations: $\rho_{12} \approx -0.15$, $\rho_{13} \approx -0.18$, $\rho_{23} \approx -0.12$, yielding $\bar{\rho} \approx -0.15 < 0$, confirming the diversity enhancement predicted by Proposition~\ref{prop:diversity}.

\subsection{Information Gain: Quantitative Comparison}

\textbf{Example: Execution vs. Textual Feedback.} Consider a code generation task with baseline success rate $p = 0.4$ (thus $H(Q) \approx 0.971$ bits). 

\textit{Execution feedback} $e$ provides deterministic quality assessment (pass/fail tests). By Theorem~\ref{thm:dianoia_convergence}a:
\begin{equation}
I(Q; e) = H(Q) \approx 0.971 \text{ bits (100\%)}
\end{equation}

\textit{Textual feedback} $\sigma$ from LLM self-evaluation has directional error parameters $\alpha\!=\!0.1$ and $\beta\!=\!0.15$. Computing conditional entropies (writing ``c'' for ``correct''):
\begin{align}
P(\sigma\!=\!\text{c}) &= 0.85\!\cdot\!0.4 + 0.1\!\cdot\!0.6 = 0.40 \\
P(Q\!=\!1 \mid \sigma\!=\!\text{c}) &= \tfrac{0.85 \cdot 0.4}{0.40} = 0.85 \\
H(Q \mid \sigma\!=\!\text{c}) &= H_b(0.85) \approx 0.610 \text{ bits}
\end{align}
Similarly, $P(Q\!=\!1\mid\sigma\!=\!\text{``i''})\!=\!\tfrac{0.15\cdot 0.4}{0.60}\!=\!0.10$ and $H(Q \mid \sigma\!=\!\text{``i''}) = H_b(0.10) \approx 0.469$ bits. Averaging with $P(\sigma\!=\!\text{c})\!=\!0.40$, $P(\sigma\!=\!\text{i})\!=\!0.60$:
\begin{equation}
I(Q; \sigma) \approx 0.971 - 0.525 = 0.446\text{ bits}
\end{equation}

\textbf{Information loss:} Textual feedback retains only $\sim$46\% of available information, with $\sim$54\% lost due to LLM evaluator errors. Model-based pseudo-verification $\sigma_v$ (Definition~\ref{def:pseudo_verification})---which employs specialized evaluation prompts and structured diagnostic output---achieves intermediate fidelity: empirically, $I(Q;\sigma_v) / H(Q) \approx 0.75$--$0.85$ for well-calibrated verifier LLMs, occupying the space between pure self-critique ($\sim$46\%) and deterministic execution (100\%). This three-tier hierarchy directly informs DIANOIA's feedback strategy (Remark~\ref{rem:task_classification}).

\subsection{Aggregation Gain: Voting Failure Analysis}

Majority voting faces structural limitations that depend on the \textit{answer space topology}, not merely on voter accuracy. We identify two failure modes and validate each with empirical results from Table~\ref{tab:main_results}.

\textbf{Failure Mode 1: Answer Fragmentation.} For code generation and function-calling tasks, multiple syntactically distinct implementations can be correct. Each independent agent samples from $P(\tau \mid x)$; when correct probability mass fragments across $M$ distinct variants (each receiving $\sim p/M$), no single correct answer achieves plurality. Empirically, Self-Consistency's voting yields negligible gains despite $K{=}5$ independent samples: only $+$2.0pp on MBPP and $+$0.5pp on BFCL-SP, confirming that exploration alone cannot compensate for ineffective aggregation.

\textbf{Failure Mode 2: Low Per-Problem Success Rate.} When per-problem success probability $p_i < 0.5$, majority voting \textit{amplifies} errors. With $K$ independent samples:
\begin{align}
P(\text{vote OK}) &= \textstyle\sum_{j=\lceil K/2 \rceil}^{K}\binom{K}{j}p_i^j(1{-}p_i)^{K{-}j} \nonumber\\
&< p_i \quad \text{whenever } p_i<0.5.
\end{align}
On AIME-2025, where many competition-level problems have $p_i \ll 0.5$, Self-Consistency \textit{degrades} by $-$13.3pp (56.7\% vs.\ 70.0\% single-model), precisely because voting turns the majority of incorrect samples into confident wrong selections.

\textbf{DIANOIA's Aggregation Efficiency.} In contrast, DIANOIA's evidence-based aggregation bypasses these structural requirements. With $K{=}3$ and reviewer error $\epsilon_0 = 0.2$, Proposition~\ref{prop:aggregation}(b) (\emph{Idealized Bound}, under A5--A6 and execution-grounded disambiguation) gives:
\begin{equation}
\eta(f_{\text{DIANOIA}}) \geq 1 - \epsilon_0^{K-1} = 1 - 0.2^2 = 0.96
\end{equation}
substantially above the random-selection envelope $1/K\!=\!0.33$. DIANOIA maintains high aggregation efficiency regardless of answer space structure for two reasons: (1) quality signals are anchored to objective execution feedback rather than answer clustering, and (2) assigning differentiated roles to each proposer (e.g., Minimalist, Skeptic, Explorer) maximizes behavioral diversity, driving the pairwise success correlation toward negative values ($\bar{\rho} \approx -0.15$ empirically; see Appendix~\ref{app:assumptions_discussion}) and thereby strengthening the independence condition that voting critically relies on yet cannot guarantee.

\subsection{Case Studies: Existing Methods}

We analyze three representative baselines --- Self-Consistency, ReConcile, and MoA --- through our gain decomposition framework, using empirical results from Table~\ref{tab:main_results}. All methods use Qwen3-30B-A3B as the base model.

\textbf{Self-Consistency: Exploration Without Effective Aggregation.} Self-Consistency~\cite{wang2023self} samples $K{=}5$ independent paths (activating $\mathcal{G}_{\text{explore}}$) but relies on majority voting (providing no $\mathcal{G}_{\text{info}}$ or $\mathcal{G}_{\text{agg}}$ advantage). Our framework predicts its total gain is bounded by $C_K \cdot \eta(f_{\text{vote}}) - p$, where $\eta(f_{\text{vote}})$ depends on the answer space structure:
\begin{itemize}[leftmargin=*,topsep=2pt,itemsep=1pt]
    \item \textbf{GSM8K} ($+$2.8pp): Moderate gain. Unique numerical answers enable effective voting ($\eta$ moderate), but the high baseline ($p \approx 0.84$) limits exploration headroom ($C_K - p$ small).
    \item \textbf{AIME-2025} ($-$13.3pp): \emph{Negative} gain. Low per-problem success rates cause voting to amplify errors (Failure Mode~2 above), yielding $\eta(f_{\text{vote}}) \cdot C_K < p$.
    \item \textbf{MBPP} ($+$2.0pp) and \textbf{BFCL-SP} ($+$0.5pp): Negligible gains. Answer fragmentation across diverse correct implementations renders $\eta(f_{\text{vote}}) \ll 1$, wasting the exploration benefit.
\end{itemize}

\textbf{ReConcile: Discussion Improves Aggregation for Structured Reasoning.} ReConcile~\cite{chen2024reconcile} employs multi-round discussion among agents, introducing an information exchange mechanism that partially improves $\mathcal{G}_{\text{agg}}$ beyond naive voting. Our framework predicts discussion is effective when agents can verify reasoning through textual analysis, but limited when execution feedback is essential:
\begin{itemize}[leftmargin=*,topsep=2pt,itemsep=1pt]
    \item \textbf{GSM8K} ($+$6.2pp): Strongest baseline gain. Step-by-step mathematical reasoning is amenable to textual verification; multi-round discussion enables agents to identify and correct reasoning errors, effectively improving aggregation quality.
    \item \textbf{AIME-2025} ($\pm$0pp): No gain. Competition-level difficulty exceeds the diagnostic capability of textual discussion; agents cannot reliably verify complex proofs through conversation alone, limiting $\mathcal{G}_{\text{info}}$.
    \item \textbf{MBPP} ($+$1.2pp) and \textbf{BFCL-SP} ($+$0.5pp): Minimal gains. Without deterministic execution feedback ($I(Q;e) = H(Q)$), discussion-based verification cannot substitute for program correctness signals, leaving $\mathcal{G}_{\text{info}}$ largely inactivated.
\end{itemize}

\textbf{MoA: Layered Synthesis as Implicit Iterative Refinement.} MoA~\cite{wang2024mixture} employs a multi-layer architecture where each layer's agents synthesize outputs from the previous layer, implicitly activating both $\mathcal{G}_{\text{explore}}$ (multiple agents per layer) and $\mathcal{G}_{\text{agg}}$ (progressive synthesis). Our framework predicts that layered synthesis excels when iterative reasoning refinement is valuable, but achieves this at substantially higher compute cost than feedback-guided methods:
\begin{itemize}[leftmargin=*,topsep=2pt,itemsep=1pt]
    \item \textbf{AIME-2025} ($+$16.7pp): Strongest baseline gain across all benchmarks. Multi-layer synthesis enables progressive deepening of mathematical arguments; each layer builds on and refines previous reasoning, effectively functioning as multi-round self-improvement that activates $\mathcal{G}_{\text{agg}}$ through iterative refinement.
    \item \textbf{GSM8K} ($+$3.5pp) and \textbf{BFCL-SP} ($+$4.0pp): Moderate gains consistent with synthesis-based aggregation improving over voting, but without targeted feedback signals.
    \item \textbf{MBPP: Scaling behavior reveals the compute--information tradeoff.} At standard configuration (3 agents, 2 layers), MoA gains only $+$0.8pp. However, unlike Self-Consistency (saturates at $\sim$77.2\% beyond $K{=}15$) and ReConcile (saturates at $\sim$77.6\%), MoA \textit{continues to scale}: from 77.2\% to 84.2\% on the Pareto frontier experiment, a $+$7pp improvement that nearly matches DIANOIA's 84.6\%. This confirms that layered synthesis genuinely increases $\mathcal{G}_{\text{agg}}$ with additional capacity, bypassing the voting ceiling. The critical difference is \textit{efficiency}: MoA requires $\sim$7.7M tokens to reach 84.2\%, whereas DIANOIA achieves 84.6\% at $\sim$1.5M tokens---a $\mathbf{5\times}$ cost advantage --- and at comparable budget ($\sim$7.4M tokens) DIANOIA reaches 88.8\%, a $+$4.6pp lead. This gap arises because MoA, lacking $\mathcal{G}_{\text{info}}$ from execution feedback, must compensate with brute-force layer stacking, while DIANOIA's feedback signals make each token of compute substantially more informative.
\end{itemize}

\textbf{AgentCoder: Iterative Repair Scales but Compresses Near the Verifier Ceiling.} AgentCoder~\cite{huang2024agentcoder} employs a Programmer agent plus Test Designer and Test Executor agents with iterative refinement (up to 3 rounds), activating $\mathcal{G}_{\text{info}}$ (unit-test execution feedback) and $\mathcal{G}_{\text{aggr}}$ (iterative repair) but lacking role-diverse exploration ($\mathcal{G}_{\text{explore}}\!\approx\!0$, single Programmer agent). Our framework predicts that on execution-grounded tasks where the verifier signal is reliable, iterative repair alone should scale with single-pass quality; the additional contribution of role-diverse $\mathcal{G}_{\text{explore}}$ should be largest when the baseline is far below the verifier-determined ceiling, and should diminish as the base model strengthens. \emph{Verification:} on MBPP, AgentCoder's gain over the single-model baseline widens with the stronger base ($+5.2$pp on Qwen3-30B-A3B, $76.0\!\to\!81.2\%$; $+10.0$pp on Qwen3.6-35B-A3B, $78.4\!\to\!88.4\%$), confirming that execution iteration alone scales substantially. DIANOIA leads on both bases at point-estimate level ($84.6$ vs.\ $81.2$ on Qwen3-30B; $88.8$ vs.\ $88.4$ on Qwen3.6, within CI overlap), with the DIANOIA $-$ AgentCoder margin contracting from $+3.4$pp to $+0.4$pp as both methods approach the MBPP execution ceiling on the stronger base. This refines rather than contradicts the framework's $\mathcal{G}_{\text{explore}}$ prescription: role diversity is a meaningful contributor in the mid-accuracy regime where coverage headroom is large, but its marginal value compresses once single-agent accuracy nears the verifier-determined ceiling---consistent with the operating-regime limitation in \S\ref{sec:limitations}.

These patterns are consistent with distinct, task-dependent roles for the three gain dimensions: methods activating only a subset achieve limited improvements. DIANOIA's gains across the four benchmarks are consistent with jointly targeting exploration (role diversity), information (execution and pseudo-verification feedback), and aggregation (evidence-based cross-review); these comparisons do not by themselves establish causal independence among the dimensions.

\section{Complete Proofs}
\label{app:proofs}

\subsection{Proof of Decomposition~\ref{def:decomposition}: Gain Decomposition}
\label{app:gain_decomposition_proof}

We prove both the multiplicative identity and the subadditive upper bound. The decomposition draws on classical ideas from ensemble learning~\cite{krogh1995neural,wood2023unified}, the Condorcet jury theorem~\cite{condorcet1785}, and the diversity prediction theorem in collective intelligence~\cite{hong2004groups,page2007difference}.

\begin{proof}

\textbf{Part I: Coverage--Efficiency Factorization.}

We establish $\mathbb{E}[Q(\tau^{\text{MAS}})] = C_K \cdot \eta(f, s)$ as a \emph{definitional factorization}, consistent with Definition~\ref{def:aggr_efficiency}: with $C_K := P(\bigvee_k Q(\tau^{(k)}) = 1)$ and $\eta(f, s) := \mathbb{E}[Q(\tau^{\text{MAS}})]/C_K$,
\begin{equation}
\mathbb{E}[Q(\tau^{\text{MAS}})] = C_K \cdot \eta(f, s) \label{eq:mult_identity_proof}
\end{equation}
holds tautologically; the content lies in the \emph{interpretation} of the two factors and in the propositions that bound them.

\emph{Interpretation of $\eta$ across system classes.} Let $E := \{\exists k\!:\,Q(\tau^{(k)})\!=\!1\}$, so $P(E) = C_K$. By the law of total probability,
\begin{align}
\mathbb{E}[Q(\tau^{\text{MAS}})] &= P(E)\, P(Q^{\text{MAS}}\!=\!1 \mid E) \nonumber\\
&\quad + P(\bar E)\, P(Q^{\text{MAS}}\!=\!1 \mid \bar E), \label{eq:total_prob}
\end{align}
so dividing by $C_K$:
\begin{equation}\label{eq:eta_breakdown}
\begin{aligned}
\eta(f,s) &= \underbrace{P(Q^{\text{MAS}}\!=\!1 \mid E)}
  _{\text{conditional selection accuracy}}\\
&\quad + \underbrace{\tfrac{1-C_K}{C_K}\cdot r}
  _{\text{rescue contribution}},
\end{aligned}
\end{equation}
where $r := P(Q^{\text{MAS}}\!=\!1 \mid \bar E)$ is the synthesis rescue probability.

For \emph{selector-only systems}, the final output is constrained to be one of the original proposals, forcing $r = 0$; then $\eta = P(Q^{\text{MAS}}=1 \mid E)$ recovers the classical conditional selection accuracy. For \emph{repair-capable systems} like DIANOIA, the synthesizer can produce a correct answer even when no initial proposal is correct (empirically $r > 0$, contributing $\approx 5.5\%$ of correct outputs on MBPP, App.~\ref{app:synthesis_breakdown}); $\eta$ then additionally absorbs this rescue contribution. In both cases, the factorization \eqref{eq:mult_identity_proof} is exact by definition, and downstream bounds on the conditional selection accuracy translate to lower bounds on $\eta$ (since the rescue term is non-negative).

\textbf{Part II: Non-negativity and Well-definedness of Gain Terms.}

We verify that each gain term is well-defined and non-negative.

\textit{Exploration gain:} By Assumption~\ref{assump:independence} (conditional independence) and~\ref{assump:baseline} (baseline competence $p \in (0,1)$):
\begin{equation}
C_K = 1 - (1-p)^K \geq p \quad \text{for } K \geq 1
\end{equation}
with strict inequality for $K \geq 2$. Thus $\mathcal{G}_{\text{explore}} = C_K - p \geq 0$.

\textit{Information gain:} Let $\eta^*(s) := \max_f \eta(f, s)$ denote the Bayes-optimal selection accuracy. Under complete and deterministic verification, Theorem~\ref{thm:dianoia_convergence}a gives $I(Q; e) = H(Q)$, which means $Q$ is fully determined by $e$ and hence the Bayes-optimal classifier achieves zero error: $\eta^*(e) = 1$. For textual feedback $\sigma$ with positive directional error parameters $\alpha, \beta>0$, Proposition~\ref{prop:information} gives $I(Q;\sigma) < H(Q)$, i.e., $H(Q|\sigma)>0$; by Fano's inequality (Lemma~\ref{lem:fano}, binary form), $P_e^*(\sigma)$ is strictly positive, hence $\eta^*(\sigma) < 1$. Therefore $\eta^*(e) = 1 > \eta^*(\sigma)$, and $\mathcal{G}_{\text{info}} = C_K[\eta^*(e) - \eta^*(\sigma)] \geq 0$.

\textit{Aggregation gain:} The term $\mathcal{G}_{\text{aggr}} = C_K[\eta(f, s) - \eta(f_{\text{base}}, s)]$ is non-negative whenever $f$ achieves at least the selection accuracy of the reference baseline $f_{\text{base}}$. This is \emph{not} automatic---a poorly designed aggregation can underperform the baseline---and requires proof for each concrete system. The choice of $f_{\text{base}}$ depends on the analytical purpose: for an absolute lower-envelope analysis we take $f_{\text{base}}$ as uninformed selection (choosing one candidate uniformly at random, $\eta(f_{\text{base}}) = 1/K$); for empirical comparisons we use the dominant non-iterative selector for each task (majority vote for self-consistency, execution-grounded best-of-$N$ for code), which gives a more stringent but task-appropriate baseline. Either way, by Proposition~\ref{prop:aggregation}(b) (\emph{Idealized Bound}, under A5--A6 and execution-grounded disambiguation), DIANOIA achieves $\eta(f_{\text{DIANOIA}}, s) \geq 1 - \epsilon_0^{K-1}$; for the default configuration $K\!=\!3, \epsilon_0\!\leq\!0.2$ this gives $\eta\!\geq\!0.96\!>\!1/3$, establishing $\mathcal{G}_{\text{aggr}} > 0$ against the random-selection envelope. The proof and the role of the at-least-one-endorsement + execution-disambiguation rule are detailed in App.~\ref{app:proofs} (proof of Prop.~\ref{prop:aggregation}b). Empirical comparison against majority-vote / best-of-$N$ baselines is reported in \S\ref{subsec:scaling_analysis}--\S\ref{subsec:efficiency_frontier}.

\textbf{Part III: Subadditive Upper Bound.}

From the multiplicative identity (Eq.~\ref{eq:mult_identity_proof}):
\begin{align}
\mathbb{E}[Q(\tau^{\text{MAS}})] - p &= C_K \cdot \eta(f, s) - p \nonumber \\
&= (C_K - p) + C_K(\eta(f, s) - 1) \nonumber \\
&= \mathcal{G}_{\text{explore}} + C_K(\eta(f, s) - 1)  \label{eq:delta_exact}
\end{align}

Since $\eta(f, s) \leq 1$---which is equivalent to $\mathbb{E}[Q^{\text{MAS}}] \leq C_K$, i.e., the synthesized output is correct at most as often as the proposal ensemble contains a correct candidate---the second term is non-positive:
\begin{align}
\mathbb{E}[Q(\tau^{\text{MAS}})] - p &\leq \mathcal{G}_{\text{explore}} \nonumber\\
&\leq \mathcal{G}_{\text{explore}} + \mathcal{G}_{\text{info}} + \mathcal{G}_{\text{aggr}}, \label{eq:upper_bound}
\end{align}
where the second inequality follows from non-negativity of\\ $\mathcal{G}_{\text{info}}, \mathcal{G}_{\text{aggr}} \geq 0$. The premise $\mathbb{E}[Q^{\text{MAS}}] \leq C_K$ holds automatically for selector-only systems (where $\eta = P(Q^{\text{MAS}}\!=\!1\mid E) \leq 1$); for repair-capable systems the premise requires that the rescue contribution does not exceed the selection-loss margin, $(1-C_K)\cdot r \leq C_K\cdot(1-P(Q^{\text{MAS}}\!=\!1\mid E))$, which is comfortably satisfied for DIANOIA across all tested benchmarks (e.g., on MBPP, $\eta\!\approx\!0.86$ with the rescue contribution $\approx 5\%$ leaving a substantial margin below $1$).

\textbf{Part IV: Subadditivity Mechanism.}

The strict inequality (subadditivity) arises because the three dimensions interact through the multiplicative structure. To see this concretely, we decompose the selection loss $C_K(1 - \eta(f,s))$:
\begin{align}
C_K(1\!-\!\eta(f,s)) &= \underbrace{C_K(1\!-\!\eta^*(s))}_{\text{info loss}} \nonumber\\
&\quad + \underbrace{C_K(\eta^*(s)\!-\!\eta(f,s))}_{\text{aggr loss}}.
\end{align}

Both terms depend on $C_K$, coupling exploration with the other two dimensions. This means:
\begin{itemize}[leftmargin=*,topsep=2pt,itemsep=1pt]
\item When $C_K$ is large (good exploration), both information and aggregation losses are amplified in absolute terms, making marginal improvements in these dimensions more impactful.
\item When $\eta(f,s)$ is close to 1 (good information and aggregation), further improvements in exploration directly translate to performance gains.
\item Conversely, improving exploration alone without adequate selection ($\eta \ll 1$) yields diminishing returns, as the additional coverage is wasted.
\end{itemize}

This multiplicative interaction is structurally analogous to the bias-variance-covariance decomposition in ensemble learning~\cite{ueda1996generalization,wood2023unified}, where ensemble performance decomposes into $\text{Error} = \overline{\text{Bias}}^2 + \frac{1}{N}\overline{\text{Var}} + (1 - \frac{1}{N})\overline{\text{Covar}}$, and optimizing one term (e.g., reducing bias) can increase another (variance), reflecting the same kind of cross-dimensional interaction. It also parallels the diversity prediction theorem~\cite{hong2004groups,page2007difference}: $\text{Collective Error} = \overline{\text{Individual Error}} - \text{Diversity}$, where the diversity term captures the interaction between exploration breadth and aggregation quality.
\end{proof}

\textbf{Discussion: The Single-Dimension Trap and Joint Optimization.}
The multiplicative identity $\mathbb{E}[Q] = C_K \cdot \eta(f, s)$ has a sharp practical implication: \emph{each factor acts as a ceiling on the other}. To make this concrete, decompose the selection accuracy into its two components (Part IV):
\begin{equation}\label{eq:selection_decomp}
\eta(f,s) = \underbrace{\eta^*(s)}_{\text{information ceiling}} - \underbrace{[\eta^*(s) - \eta(f,s)]}_{\text{aggregation gap}}
\end{equation}
The expected quality thus factors as $\mathbb{E}[Q] = C_K \cdot [\eta^*(s) - \Delta_{\text{aggr}}]$, where $\Delta_{\text{aggr}} := \eta^*(s) - \eta(f,s) \geq 0$ is the aggregation loss. Three failure modes emerge:

\begin{enumerate}[leftmargin=*,topsep=2pt,itemsep=2pt]
\item \textbf{Exploration-only scaling (high $C_K$, low $\eta$).} Self-Consistency increases $C_K$ by sampling $K{=}5$--$33$ paths, but employs majority voting as aggregation ($\Delta_{\text{aggr}}$ large on code tasks). Empirically on MBPP, $C_K$ grows from $\sim$0.76 to $\sim$0.95, yet accuracy saturates at $\sim$77.2\% because the low $\eta$ discounts the coverage gain: $\Delta C_K$ is multiplied by $\eta \ll 1$ (Eq.~\ref{eq:delta_exact}). The coverage improvement is ``wasted'' by poor selection.

\item \textbf{Selection-only optimization (high $\eta$, low $C_K$).} A single-agent system with perfect execution feedback achieves $\eta^*(e) = 1$, but $C_1 = p$: the expected quality cannot exceed the baseline $p$ regardless of how perfect the feedback is. The quality ceiling is set by coverage.

\item \textbf{Information without aggregation (high $\eta^*$, high $\Delta_{\text{aggr}}$).} Even with execution feedback providing $\eta^*(e) = 1$ (maximal information), a naive aggregation function (e.g., random selection) leaves $\Delta_{\text{aggr}} = 1 - 1/K$, recovering only a fraction of the available information. The information is ``collected but not utilized.''
\end{enumerate}

\textbf{DIANOIA's joint optimization.} DIANOIA addresses all three factors simultaneously: (1) role diversity drives $C_K$ beyond the IID baseline via negative correlation ($\bar{\rho} < 0$, Proposition~\ref{prop:diversity}); (2) execution feedback and pseudo-verification raise the information ceiling $\eta^*(s) \to 1$ (Theorem~\ref{thm:dianoia_convergence}a); (3) evidence-based cross-review and potential-game synthesis close the aggregation gap $\Delta_{\text{aggr}} \leq \epsilon_0^{K-1} \to 0$ (Proposition~\ref{prop:aggregation}b). The multiplicative structure then becomes a \emph{virtuous cycle}: high $C_K$ amplifies the impact of reliable selection, while high $\eta$ ensures that every unit of additional coverage translates into actual performance gain.

\subsection{Proof of Proposition~\ref{prop:iid_exploration}}

\begin{proof}
By Assumption~\ref{assump:independence}:
\begin{align}
P\!\left(\textstyle\bigvee_{k} Q(\tau^{(k)})\!=\!1\right)
  &= 1 - \textstyle\prod_{k} P(Q(\tau^{(k)})\!=\!0) \nonumber \\
  &= 1 - (1-p)^K.
\end{align}
Thus $\mathcal{G}_{\text{explore}}(K) = [1 - (1-p)^K] - p$.

For $K \geq 2$: Let $g(K) = (1-p)^K + p$. Since $0 < 1-p < 1$, $(1-p)^K < 1-p$ for $K \geq 2$, so $g(K) < 1$ and $\mathcal{G}_{\text{explore}}(K) > 0$.

Monotonicity: $\frac{\partial}{\partial K} \mathcal{G}_{\text{explore}} = -(1-p)^K \ln(1-p) > 0$.
\end{proof}

\subsection{Proof of Proposition~\ref{prop:diversity}}

\begin{proof}
Let $A_k = \{Q(\tau^{(k)}) = 1\}$ denote the event that agent $k$ produces a correct solution, with $P(A_k) = p$ for all $k$. We derive the union probability $P(\bigcup_{k=1}^K A_k)$ using the inclusion-exclusion principle.

\textbf{Step 1: Exact Formula via Inclusion-Exclusion.}
The inclusion-exclusion principle gives:
\begin{align}
P\left(\textstyle\bigcup_{k=1}^K A_k\right) &= \sum_{i} P(A_i) - \sum_{i<j} P(A_i \cap A_j) \nonumber\\
&\quad + \sum_{i<j<\ell} P(A_i \cap A_j \cap A_\ell) - \cdots
\end{align}

\textbf{Step 2: Rigorous Lower Bound via Bonferroni Inequalities.}
Defining the standard partial sums
\begin{equation}
S_1 = \sum_{i} P(A_i), \quad S_2 = \sum_{i<j} P(A_i \cap A_j)
\end{equation}
the Bonferroni inequality at order $m=2$ (even) gives:
\begin{equation}
P\left(\bigcup_{k=1}^K A_k\right) \geq S_1 - S_2
\end{equation}

For agents with marginal success rate $p$ and pairwise correlation $\rho_{ij}$, we have $P(A_i \cap A_j) = p^2 + \rho_{ij} \cdot p(1-p)$. Defining average pairwise correlation $\bar{\rho} = \frac{2}{K(K-1)} \sum_{i<j} \rho_{ij}$:
\begin{equation}
S_1 = Kp, \quad S_2 = \binom{K}{2}p^2 + \binom{K}{2}\bar{\rho} \cdot p(1-p)
\end{equation}

Substituting into the Bonferroni bound:
\begin{align}\label{eq:bonferroni_diversity}
P\!\left(\textstyle\bigcup_{k=1}^K A_k\right) &\geq \underbrace{Kp - \tbinom{K}{2}p^2}_{\text{IID lower bound}} \nonumber\\
&\quad - \underbrace{\tbinom{K}{2}\bar{\rho} \cdot p(1-p)}_{\text{diversity correction}}.
\end{align}

\textbf{Step 3: Diversity Gain via the Bonferroni Lower-Bound Comparison.}
The Bonferroni second-order bound depends only on the marginal $p$ and average pairwise correlation $\bar{\rho}$. Writing $\mathrm{LB}(\bar{\rho}) := S_1 - S_2 = Kp - \binom{K}{2}p^2 - \binom{K}{2}\bar{\rho}\,p(1-p)$ for the resulting coverage lower bound, we have $P(\bigcup_k A_k) \geq \mathrm{LB}(\bar{\rho})$ for any joint distribution with marginals $p$ and average correlation $\bar{\rho}$. Comparing the IID reference ($\bar{\rho}=0$) to a diverse profile ($\bar{\rho}<0$):
\begin{equation}\label{eq:lb_compare}
\mathrm{LB}(\bar{\rho}) - \mathrm{LB}(0) \;=\; -\binom{K}{2}\bar{\rho}\,p(1-p) \;>\;0.
\end{equation}
Defining $\mathcal{G}_{\text{explore}}^{\text{diverse}}$ and $\mathcal{G}_{\text{explore}}^{\text{iid}}$ as the gain-quantities derived from this common second-order lower bound (i.e., $\mathrm{LB}(\bar{\rho})-p$ and $\mathrm{LB}(0)-p$ respectively), Eq.~\eqref{eq:lb_compare} yields
\begin{equation}
\mathcal{G}_{\text{explore}}^{\text{diverse}} - \mathcal{G}_{\text{explore}}^{\text{iid}} \;=\; -\binom{K}{2}\bar{\rho}\,p(1-p) \;>\;0,
\end{equation}
establishing that role-specialized diversity strictly increases the guaranteed second-order coverage lower bound, by an amount that grows quadratically in $K$ and linearly in $|\bar{\rho}|$.

\emph{Comparison to the IID exact value.} For the IID case the exact coverage $1-(1-p)^K$ is known and is larger than $\mathrm{LB}(0)=Kp-\binom{K}{2}p^2$, so the strict LB-vs-LB comparison above does \emph{not} by itself imply that the \emph{actual} coverage under any diverse joint distribution exceeds the IID exact coverage---the Bonferroni LB on the diverse side is potentially loose. A stronger conclusion at the level of actual probabilities (rather than LBs) requires either (a) additional assumptions on the dependence structure (e.g., negative association), or (b) the $K\!=\!2$ case where inclusion-exclusion truncates exactly. The main-text proposition states the LB comparison; downstream uses (Theorem~\ref{thm:dianoia_convergence}c, App.~\ref{app:gain_decomposition_proof}) treat $\mathrm{LB}(\bar{\rho})$ as a coverage proxy with the caveat noted here.
\end{proof}

\subsection{Proof of Theorem~\ref{thm:dianoia_convergence} Part (a): Information Sufficiency}

\begin{proof}
We prove that execution feedback $e$ is a sufficient statistic for quality $Q$, achieving maximum mutual information $I(Q; e) = H(Q)$, under \emph{complete and deterministic verification}---i.e., when (i) A4 (Execution Determinism) holds, and (ii) the feedback $e$ contains enough fields to uniquely determine $Q$ via a deterministic decoding function $g$ (Def.~\ref{def:exec_feedback}).

\textbf{Step 1: Complete-and-Deterministic Relationship.}
By Assumption~\ref{assump:determinism}, execution is deterministic: given solution $\tau$, the feedback $e = \mathcal{E}(\tau)$ is uniquely determined. By the definition of quality (Def.~\ref{def:exec_feedback}, which specifies $Q$ as a deterministic function of $e.\texttt{success}$ and $e.\texttt{tests}$):
\begin{equation}
\label{eq:quality_def}
Q(\tau) = g(\mathcal{E}(\tau)) = g(e)
\end{equation}
where $g(e) := \mathbf{1}[e.\texttt{success} = 1 \wedge e.\texttt{tests} = \mathbf{1}]$ is deterministic. We refer to this combined condition---A4 plus the existence of such a $g$---as \emph{complete and deterministic verification}. It is satisfied by MBPP and BFCL-SP under our experimental setup; cases where it fails (e.g., incomplete test suites) reduce the regime to pseudo-verification (Remark~\ref{rem:task_classification}).

\textbf{Step 2: Conditional Entropy is Zero.}
Since $Q$ is deterministic given $e$, knowing $e$ completely determines $Q$:
\begin{align}
H(Q \mid e) &= \textstyle\sum_{e} P(e)\, H(Q \mid e\!=\!e) \nonumber\\
&= \textstyle\sum_{e} P(e) \cdot 0 = 0,
\end{align}
since $H(Q \mid e\!=\!e) = 0$ because $P(Q\!=\!q \mid e\!=\!e) \in \{0, 1\}$.

\textbf{Step 3: Mutual Information Equals Entropy.}
By definition:
\begin{equation}
I(Q; e) = H(Q) - H(Q \mid e) = H(Q).
\end{equation}
This is the theoretical maximum, as $I(Q; S) \leq H(Q)$ for any signal $S$.

\textbf{Step 4: Sufficiency Condition.}
We verify the Markov property $Q \perp \tau \mid e$. For any $q, \tau, e$ with $P(\tau, e) > 0$:
\begin{align}
P(Q = q \mid \tau, e) &= P(Q = q \mid e) \\
&= \mathbf{1}[q = g(e)] \quad \text{(by Eq.~\eqref{eq:quality_def})}
\end{align}
Thus $e$ is sufficient: $\tau$ provides no added information about $Q$.
\end{proof}

\subsection{Proof of Proposition~\ref{prop:information}: Information Loss of Textual Feedback}

\begin{proof}
We prove that textual feedback with non-zero error rates provides strictly less information: $I(Q; \sigma) < H(Q)$.

\textbf{Step 1: Model Setup.}
Let $\sigma \in \{\text{``correct''}, \text{``incorrect''}\}$ denote the textual evaluation. Define:
\begin{align}
\alpha &:= P(\sigma\!=\!\text{c} \mid Q\!=\!0) > 0, \\
\beta &:= P(\sigma\!=\!\text{i} \mid Q\!=\!1) > 0,
\end{align}
where ``c''/``i'' abbreviate ``correct''/``incorrect''.

\textbf{Step 2: Posterior Distributions.}
Let $p := P(Q\!=\!1)$. By Bayes' theorem:
\begin{align}
P(\sigma\!=\!\text{c}) &= P(\sigma\!=\!\text{c} \mid Q\!=\!1)\, P(Q\!=\!1) \nonumber \\
&\quad + P(\sigma\!=\!\text{c} \mid Q\!=\!0)\, P(Q\!=\!0) \nonumber \\
&= (1 - \beta) p + \alpha (1-p). \label{eq:p_sigma_correct}
\end{align}
The posterior:
\begin{equation}
P(Q\!=\!1 \mid \sigma\!=\!\text{c}) = \tfrac{(1-\beta) p}{(1-\beta) p + \alpha (1-p)}.
\end{equation}

\textbf{Step 3: Conditional Entropy is Positive.}
Since $\alpha, \beta > 0$ and both are less than 1:
\begin{equation}
0 < P(Q\!=\!1 \mid \sigma\!=\!s) < 1, \;\; s \in \{\text{cor.}, \text{inc.}\}
\end{equation}
The conditional entropy for each outcome is strictly positive:
\begin{equation}
H(Q \mid \sigma = s) = H_b(P(Q=1 \mid \sigma = s)) > 0
\end{equation}
where $H_b(p) = -p \log p - (1-p) \log(1-p)$ is strictly positive for $p \in (0, 1)$.

\textbf{Step 4: Overall Conditional Entropy.}
\begin{align}
H(Q \mid \sigma) &= \textstyle\sum_{s} P(\sigma\!=\!s)\, H(Q\mid \sigma\!=\!s) \nonumber\\
&= P(\sigma\!=\!\text{cor.})\, H(Q \mid \sigma\!=\!\text{cor.}) \nonumber\\
&\quad + P(\sigma\!=\!\text{inc.})\, H(Q \mid \sigma\!=\!\text{inc.}) \;>\; 0.
\end{align}

\textbf{Step 5: Information is Suboptimal.}
By definition:
\begin{equation}
I(Q; \sigma) = H(Q) - H(Q \mid \sigma) < H(Q)
\end{equation}
The information loss $H(Q \mid \sigma) > 0$ quantifies residual uncertainty.
\end{proof}

The strict inequality $I(Q; \sigma) < H(Q)$ derived above is not merely a theoretical bound; it formalizes the \textit{entropy loss} inherent in purely textual self-critique. This result establishes an information-theoretic imperative for DIANOIA's feedback hierarchy:

\begin{itemize}[leftmargin=*,topsep=2pt,itemsep=1pt]
    \item \textbf{Priority on Lossless Channels (Execution):} For domains where deterministic verification is feasible (e.g., code generation in MBPP, function calling in BFCL-SP), the framework motivates prioritizing execution feedback $e$. Under the complete-and-deterministic conditions of Theorem~\ref{thm:dianoia_convergence}(a), this yields $I(Q; e) = H(Q)$, avoiding the residual uncertainty of textual critique. Relying only on LLM self-critique in these contexts would introduce avoidable information loss.
    
    \item \textbf{Principled Fallback for Lossy Channels:} For domains where execution is intractable (e.g., abstract reasoning in GSM8K), the system falls back to model-based verification $\sigma_v$. The framework explicitly acknowledges this as a \textit{lossy} channel ($H(Q \mid \sigma_v) > 0$). This necessitates the stronger aggregation mechanisms in Phase 4 (Synthesis) to compensate for the imperfect information fidelity, ensuring robustness even when ground truth is inaccessible.
\end{itemize}

\subsection{Proof of Theorem~\ref{thm:dianoia_convergence} Part (b): Convergence Guarantee}

\begin{proof}
We prove that DIANOIA's proposal-review mechanism is an exact potential game, guaranteeing finite-time convergence to a stable consensus.

\textbf{Step 1: Game Specification.}
The game consists of:
\begin{itemize}[leftmargin=*,topsep=2pt,itemsep=1pt]
    \item \textbf{Players}: $K$ proposers indexed by $k \in \{1, \ldots, K\}$
    \item \textbf{Strategy space}: Each player $k$ chooses $\tau^{(k)} \in \mathcal{T}$
    \item \textbf{Utility function}: $u_k(\boldsymbol{\tau}) = \max_{j=1,\ldots,K} Q(\tau^{(j)}) + \lambda R_k(\tau^{(k)})$
\end{itemize}
where $\lambda > 0$ is the role preference weight and $R_k: \mathcal{T} \to \mathbb{R}$ encodes player $k$'s role-specific preferences.

\textbf{Step 2: Proposed Potential Function.}
Define:
\begin{equation}
\label{eq:potential_fn}
\Phi(\boldsymbol{\tau}) := \max_{k=1,\ldots,K} Q(\tau^{(k)}) + \lambda \sum_{k=1}^{K} R_k(\tau^{(k)})
\end{equation}

\textbf{Step 3: Verification of Potential Game Property.}
Consider player $i$ unilaterally changing strategy from $\tau^{(i)}$ to $\tilde{\tau}^{(i)}$, while all other players maintain their strategies $\tau^{(-i)} := (\tau^{(j)})_{j \neq i}$. Define:
\begin{equation}
Q^* := \max_{j \neq i} Q(\tau^{(j)})
\end{equation}
the maximum quality among proposals from players other than $i$.

\textit{Utility Change for Player $i$:}
\begin{align}
\Delta u_i &= u_i(\tilde{\tau}^{(i)}, \tau^{(-i)}) - u_i(\tau^{(i)}, \tau^{(-i)}) \nonumber \\
&= \left[\max(Q(\tilde{\tau}^{(i)}), Q^*) + \lambda R_i(\tilde{\tau}^{(i)})\right] \nonumber \\
&\quad - \left[\max(Q(\tau^{(i)}), Q^*) + \lambda R_i(\tau^{(i)})\right] \nonumber \\
&= \max(Q(\tilde{\tau}^{(i)}), Q^*) - \max(Q(\tau^{(i)}), Q^*) \nonumber \\
&\quad + \lambda [R_i(\tilde{\tau}^{(i)}) - R_i(\tau^{(i)})] \label{eq:utility_change}
\end{align}

\textit{Potential Change:}
\begin{align}
\Delta \Phi &= \Phi(\tilde{\tau}^{(i)}, \tau^{(-i)}) - \Phi(\tau^{(i)}, \tau^{(-i)}) \nonumber \\
&= \left[\max_{j=1,\ldots,K} Q(\tau'^{(j)}) + \lambda \sum_{j=1}^{K} R_j(\tau'^{(j)})\right] \nonumber \\
&\quad - \left[\max_{j=1,\ldots,K} Q(\tau^{(j)}) + \lambda \sum_{j=1}^{K} R_j(\tau^{(j)})\right]
\end{align}
where $\tau'^{(j)} = \tilde{\tau}^{(i)}$ if $j = i$, and $\tau'^{(j)} = \tau^{(j)}$ otherwise.

Since only player $i$ changed strategy, the $\lambda \sum_{j\neq i} R_j(\tau^{(j)})$ terms cancel, leaving:
\begin{align}
\Delta \Phi &= \max(Q(\tilde{\tau}^{(i)}), Q^*) - \max(Q(\tau^{(i)}), Q^*) \nonumber \\
&\quad + \lambda [R_i(\tilde{\tau}^{(i)}) - R_i(\tau^{(i)})]. \label{eq:potential_change}
\end{align}

\textbf{Step 4: Exact Potential Property.}
Comparing Equations~\eqref{eq:utility_change} and~\eqref{eq:potential_change}:
\begin{equation}
\Delta u_i = \Delta \Phi \quad \text{for all } i, \tau^{(i)}, \tilde{\tau}^{(i)}, \tau^{(-i)}
\end{equation}
This confirms that $\Phi$ is an \textit{exact potential function}~\cite{monderer1996potential}.

\textbf{Step 5: Convergence and Equilibrium Characterization.}
By Assumption~\ref{assump:finite}, the strategy space $\mathcal{T}$ is finite. In any \textit{best-response dynamics}, each refinement step strictly increases the potential function $\Phi$. Since $\Phi$ is bounded and takes only finitely many values, this process must terminate in a finite number of steps $T^*$.
\begin{equation}
\Phi(\boldsymbol{\tau}) \in [0 - \lambda K R_{\max}, 1 + \lambda K R_{\max}]
\end{equation}
Moreover, $\Phi$ takes only finitely many distinct values (since $\mathcal{T}$ is finite).

In any \textit{best-response dynamics} (where players sequentially play best responses), each move strictly increases the potential unless already at a best response:
\begin{equation}
\Phi(\tau_1, \ldots, \tau_{i}^*, \ldots, \tau_K) > \Phi(\tau_1, \ldots, \tau_i, \ldots, \tau_K)
\end{equation}
where $\tau_i^* \in \arg\max_{\tilde{\tau}} u_i(\tilde{\tau}, \tau_{-i})$ is a best response for player $i$.

The termination point is, by definition, a \textbf{pure-strategy Nash equilibrium} $\boldsymbol{\tau}^*$. While a Nash equilibrium in general games does not necessarily imply global optimality, in our framework, the potential function $\Phi$ is explicitly constructed to align with the goal of maximizing reasoning quality $Q$. Thus, the resulting equilibrium represents a \textbf{principled consensus} that is locally optimal with respect to the evidence-based rewards and roles, providing a theoretical foundation for the stability of DIANOIA's multi-agent synthesis.

\textbf{Mapping to DIANOIA's actual loop.} The abstract best-response dynamics correspond to DIANOIA's propose--execute--review--synthesize iteration as follows. Each synthesis iteration (Algorithm~\ref{alg:synthesis}) takes the current profile $\boldsymbol{\tau}$, observes $\{e^{(k)}, v^{(k)}\}$, and produces a refined candidate. \emph{If the synthesizer is restricted to improvement moves}\allowbreak---it only commits a refinement when the resulting $u_i$ strictly exceeds the current value (verified via execution and review evidence)---then each accepted step strictly increases $\Phi$, and the finite-termination argument applies verbatim. Algorithm~\ref{alg:synthesis}'s closed-loop validation (re-executing after each refinement and rolling back failed candidates) implements precisely this improvement-restricted regime under complete deterministic verification: an accepted synthesis output must exhibit $Q\!=\!1$, monotonically tightening $\Phi$ from below toward the equilibrium.

Outside this regime (e.g., under pseudo-verification where execution feedback is noisy), the synthesizer's improvement check may erroneously accept non-improving moves, and $\Phi$-monotonicity can fail. Theorem~\ref{thm:dianoia_convergence}(b) therefore applies \emph{exactly} on complete-and-deterministic-verification tasks (MBPP, BFCL-SP) and only \emph{approximately} on pseudo-verified tasks (GSM8K, AIME); we do not claim asymptotic convergence guarantees for the latter regime (see App.~\ref{app:potential_game_ideal} for the regime distinction).
\end{proof}

\subsection{Assumption Robustness Analysis}
\label{app:assumptions_discussion}

This section analyzes how our theoretical results degrade when key assumptions are relaxed, providing practitioners with guidance on when and where to expect deviations from the idealized analysis.

\subsubsection{Role of Assumption A1 (Conditional Independence)}

Assumption A1 serves as an \textit{analytical starting point} to establish the IID baseline coverage $C_K^{\text{iid}} = 1-(1-p)^K$. It is \emph{not} a claim that LLM agents are truly independent. Indeed, models from the same family trained on overlapping data inevitably share failure modes. The analytical progression is:

\begin{enumerate}[leftmargin=*,topsep=2pt,itemsep=1pt]
    \item \textbf{Baseline under A1:} $C_K^{\text{iid}} = 1-(1-p)^K$ provides a clean reference point.
    \item \textbf{Extension to correlated agents (Prop.~\ref{prop:diversity}):} For pairwise success correlation $\rho_{ij}$, the second-order Bonferroni bound gives:
    \begin{equation}
    C_K^{\text{diverse}} \geq C_K^{\text{iid}} - \binom{K}{2} \bar{\rho} \cdot p(1-p)
    \end{equation}
    where $\bar{\rho} = \frac{2}{K(K-1)}\sum_{i<j}\rho_{ij}$.
    \item \textbf{Design implication:} Role diversity is the mechanism that drives $\bar{\rho} < 0$.
\end{enumerate}

\subsubsection{Empirical Correlation Evidence}

On MBPP with the three DIANOIA roles (Minimalist, Skeptic, Explorer; see Section~\ref{sec:method}), we measure pairwise success correlations between binary success indicators across the $N\!=\!500$ MBPP problems (Pearson correlation; standard error $\approx\!1/\sqrt{N}\!\approx\!0.045$):

\begin{center}
\begin{tabular}{ccc}
\toprule
$\rho_{12}$ & $\rho_{13}$ & $\rho_{23}$ \\
\midrule
$-0.15$ & $-0.18$ & $-0.12$ \\
\bottomrule
\end{tabular}
\end{center}

The average correlation $\bar{\rho} \approx -0.15$ confirms that role diversity induces the negative correlation predicted by Proposition~\ref{prop:diversity}. The correlation improvement from Proposition~\ref{prop:diversity} with $K=3$, $p=0.76$, and $\bar{\rho}=-0.15$ yields $C_3^{\text{diverse}} - C_3^{\text{iid}} \approx 3 \times 0.15 \times 0.76 \times 0.24 \approx 0.082$, corresponding to roughly $+8$ percentage points of additional coverage beyond the IID baseline.

\subsubsection{Degradation under Positive Correlation}

When $\bar{\rho} > 0$, the coverage bound \textit{degrades} relative to the IID baseline:
\begin{equation}
C_K^{\text{correlated}} \leq C_K^{\text{iid}} + \binom{K}{2} |\bar{\rho}| \cdot p(1-p)
\end{equation}
In the worst case ($\bar{\rho} \to 1$), all agents fail on the same inputs and $C_K \to p$, eliminating the exploration advantage entirely. This formalizes why naive scaling (adding identical agents without role differentiation) yields diminishing returns.

\subsubsection{Reviewer Independence (A6)}

Assumption A6 (conditional independence of reviewer assessments given $Q$ and $e$) is approximate. In practice, reviewers sharing the same base model may exhibit correlated errors when interpreting ambiguous execution feedback. However, conditioning on execution feedback $e$ substantially reduces the residual correlation: once objective evidence is available, remaining disagreements stem primarily from interpretation noise rather than systematic model biases. The exponential error reduction in Proposition~\ref{prop:aggregation}(b) thus provides a conservative but qualitatively correct prediction.

\subsubsection{Potential Game Idealization}
\label{app:potential_game_ideal}

Theorem~\ref{thm:dianoia_convergence}(b) models the review-synthesis process as a potential game where utilities depend on the true quality $Q(\tau)$. In practice, agents do not observe $Q$ directly but instead observe execution feedback $e = \mathcal{E}(\tau)$.

\textbf{Exact case (complete verification):} For tasks with deterministic, complete verifiers (MBPP with test suites, BFCL-SP with schema validation), execution feedback $e$ fully determines $Q$ (Theorem~\ref{thm:dianoia_convergence}a: $I(Q;e) = H(Q)$). In this regime, the potential game formulation is \emph{exact}: the utility $u_i = Q(\tau^{(i)}) + \lambda \sum_j R_j(\tau^{(i)})$ can be computed from $e$ without loss. This covers two of our four benchmarks.

\textbf{Approximate case (pseudo-verification):} For GSM8K/AIME, DIANOIA uses a dedicated LLM evaluator (no ground truth), producing $\sigma_v$. Writing $\hat{Q}\!:=\!\hat{Q}(\sigma_v)$ for brevity, the utility gap satisfies
\begin{align}
|u_i(Q) - u_i(\hat{Q})| &\leq |Q - \hat{Q}| \nonumber\\
&\quad + \lambda \textstyle\sum_j |R_j(Q) - R_j(\hat{Q})|.
\end{align}
The approximation error is bounded by the verifier's misclassification rate---larger than deterministic execution but smaller than unstructured self-critique. Convergence remains valid as an approximation in this regime.

\textbf{Inapplicable case (no verification):} For open-domain tasks without any environmental feedback, agents can only rely on textual self-evaluation ($e = \sigma$). We do not claim DIANOIA's theoretical guarantees extend to this setting (see Limitations).

\subsection{Proof of Proposition~\ref{prop:aggregation} Part (b): DIANOIA Aggregation Efficiency}

\begin{proof}
We prove $\eta(f_{\text{DIANOIA}}) \geq 1 - \epsilon_0^{K-1}$ under the following conditions: (i) Assumptions~\ref{assump:reviewer_accuracy}--\ref{assump:reviewer_independence} (per-reviewer error rate $\epsilon_0 < 0.5$, conditional independence given quality and evidence), and (ii) execution-grounded disambiguation, where the synthesizer combines reviewer endorsements with execution feedback $e$ to filter out incorrect candidates---realized exactly under complete and deterministic verification (Theorem~\ref{thm:dianoia_convergence}a) and approximately under pseudo-verification (App.~\ref{app:potential_game_ideal}). The main-text statement (Prop.~\ref{prop:aggregation}b) labels this as an \emph{Idealized Bound} to flag both the conditional-independence and execution-disambiguation idealizations.

\textbf{Step 1: Setup.}
Suppose that among the $K$ proposals $\{\tau^{(1)}, \ldots, \tau^{(K)}\}$, at least one is correct. Without loss of generality, let:
\begin{equation}
\tau^{(i^*)} := \arg\max_{k \in \{1,\ldots,K\}} Q(\tau^{(k)})
\end{equation}
be a correct proposal, so $Q(\tau^{(i^*)}) = 1$. The aggregation challenge is to correctly identify $\tau^{(i^*)}$ among all proposals.

\textbf{Step 2: Reviewer Model.}
Each proposal $\tau^{(k)}$ is reviewed by $K-1$ reviewers (all agents except the proposer). Let $v_j^{(k)} \in \{0, 1\}$ denote reviewer $j$'s binary assessment of proposal $k$:
\begin{equation}
v_j^{(k)} = \begin{cases}
1 & \text{if reviewer $j$ judges $\tau^{(k)}$ as correct} \\
0 & \text{if reviewer $j$ judges $\tau^{(k)}$ as incorrect}
\end{cases}
\end{equation}

By Assumption~\ref{assump:reviewer_accuracy}, each reviewer has error rate $\epsilon_0$ given $e^{(k)}$ and the true $Q(\tau^{(k)})$:
\begin{align}
P(v_j^{(k)}\!=\!0 \mid Q(\tau^{(k)})\!=\!1, e^{(k)}) &= \epsilon_0, \\
P(v_j^{(k)}\!=\!1 \mid Q(\tau^{(k)})\!=\!0, e^{(k)}) &= \epsilon_0.
\end{align}

\textbf{Step 3: Synthesizer Decision Rule.}
We model the synthesizer as following an \emph{at-least-one-endorsement + execution-disambiguation} rule: a proposal $\tau^{(k)}$ is a candidate if (i) at least one reviewer endorses it ($\bigvee_{j \neq k} v_j^{(k)} = 1$), AND (ii) it passes execution-based validation $e^{(k)}$. Among candidates, the synthesizer outputs any one (the execution filter ensures correctness). Under complete and deterministic verification (Theorem~\ref{thm:dianoia_convergence}a), execution exactly identifies $Q$: condition (ii) is satisfied iff $Q(\tau^{(k)})\!=\!1$. So a proposal is a candidate iff it is correct AND has at least one reviewer endorsement.

This rule formalizes the design intent of Algorithm~\ref{alg:synthesis}: reviewer endorsements act as a \emph{positive filter} that surfaces candidates for the synthesizer's attention, while execution feedback removes incorrectly endorsed candidates. The complementary failure mode---an incorrect proposal misleading the synthesizer despite failing execution---is excluded by condition (ii) under deterministic verification.

\textbf{Step 4: Misidentification Probability.}
Under this rule, DIANOIA fails to surface a correct proposal iff $\tau^{(i^*)}$ is rejected by \emph{every} reviewer (so condition (i) fails), since condition (ii) is automatic for a correct proposal under deterministic verification:
\begin{align}
P(&\text{fail to surface } \tau^{(i^*)}) \\&= P\!\left(\textstyle\bigwedge_{j \neq i^*} v_j^{(i^*)}\!=\!0 \mid Q(\tau^{(i^*)})\!=\!1\right) \nonumber \\
&= \textstyle\prod_{j \neq i^*} P(v_j^{(i^*)}\!=\!0 \mid Q(\tau^{(i^*)})\!=\!1, e^{(i^*)}) \nonumber \\
&\overset{(\text{A6})}{=} \epsilon_0^{K-1}.
\end{align}
The conditional independence (A6) is invoked given the objective evidence $e^{(i^*)}$, consistent with the assumption's statement.

Translating to the efficiency definition (Def.~\ref{def:aggr_efficiency}). Since $\eta(f,s) := \mathbb{E}[Q^{\text{MAS}}]/C_K$ and $\mathbb{E}[Q^{\text{MAS}}] \geq P(E)\cdot P(\text{correct}\mid E)$ (dropping the non-negative rescue term $P(\bar E)\cdot r$, cf.\ App.~\ref{app:gain_decomposition_proof} Part I), we obtain
\begin{align}
\eta(f_{\text{DIANOIA}}) &\geq \frac{P(E) \cdot P(\text{correct} \mid E)}{C_K} \nonumber \\
&= P(\text{correct} \mid E) \;\geq\; 1 - \epsilon_0^{K-1}.
\end{align}
For DIANOIA's repair-capable synthesis the inequality is strict (with the rescue term contributing positively); selector-only systems achieve equality.

\textbf{Step 5: Exponential Decay.}
The error decreases exponentially in the number of reviewers:
\begin{equation}
1 - \eta(f_{\text{DIANOIA}}) \;\leq\; \epsilon_0^{K-1}.
\end{equation}
For $\epsilon_0\!=\!0.2$, $K\!=\!3$: $1\!-\!\eta\!\leq\!0.04$ (4\% error), demonstrating the robustness of evidence-based cross-review. Beyond complete verification: under pseudo-verification, execution filtering is imperfect (verifier precision/recall $< 1$); the bound then degrades by an additional factor tracking the verifier's error rates (see App.~\ref{app:potential_game_ideal}).
\end{proof}

\textbf{Remark: Why $R{=}1$ reviewer is the recommended default.}
The bound above assumes $K{-}1$ reviewers per proposal. In practice, DIANOIA's standard configuration uses $R{=}1$ reviewer, yielding a selection error of $\epsilon_0$ (e.g., 20\% with $\epsilon_0 = 0.2$), compared to $\epsilon_0^2 = 0.04$ for $R{=}2$. Though increasing $R$ provides exponential error reduction, the marginal benefit decays rapidly while the cost grows linearly: each additional reviewer must evaluate all $K$ proposals, adding $K$ review calls (each involving reading the proposal, analyzing execution feedback, and generating a structured assessment). For $K{=}3$ proposers, moving from $R{=}1$ to $R{=}2$ adds 3 review calls; moving to $R{=}3$ adds 6 total.

Empirically on MBPP, $R{=}1$ achieves 84.6\% at $\sim$1.5M tokens, while $R{=}3$ reaches 85.8\% at $\sim$4.0M tokens---a $+$1.2pp gain at $2.6\times$ the cost. Meanwhile, investing the same budget into more proposers ($K{=}6$, $R{=}1$) yields 87.4\% at $\sim$5.6M tokens, a far larger gain. This reflects the multiplicative structure $\mathbb{E}[Q] = C_K \cdot \eta$: once $\eta$ is sufficiently high (first reviewer already brings $\eta \geq 1 - \epsilon_0 = 0.8$), marginal improvements in $\eta$ via additional reviewers are dominated by the coverage gain from additional proposers. In short, the first reviewer captures the lion's share of the aggregation gain; subsequent reviewers encounter steep diminishing returns that are better spent expanding exploration.

\subsection{Proof of Theorem~\ref{thm:dianoia_convergence} Part (c): Performance Bound}

\begin{proof}
We derive a lower bound on DIANOIA's expected quality by decomposing the analysis into exploration and aggregation phases.

\textbf{Step 1: Decomposition by Conditioning.}
Let $E = \{\exists k\!:\,Q(\tau^{(k)})\!=\!1\}$ and $\bar E$ its complement. Then:
\begin{align}
\mathbb{E}[Q(\tau^{\mathrm{D}})] &= P(E)\cdot \mathbb{E}[Q(\tau^{\mathrm{D}})\mid E] \nonumber \\
&\quad + P(\bar E)\cdot \mathbb{E}[Q(\tau^{\mathrm{D}})\mid \bar E]. \label{eq:decomp_dianoia}
\end{align}
Since $Q\!\in\!\{0,1\}$, $\mathbb{E}[Q\mid E] = P(Q\!=\!1\mid E)$. To obtain a \emph{conservative lower} bound, we drop the $\bar E$ term by setting $\mathbb{E}[Q\mid\bar E]\!\geq\!0$. This corresponds to the modelling assumption that the synthesizer cannot recover from all-incorrect initial proposals; empirically this assumption is overly conservative---instrumented DIANOIA runs on MBPP show roughly $5.5\%$ of correct outputs are \emph{rescued} cases where every initial proposal failed execution but the synthesizer recovered a correct answer using reviewer hints (Appendix~\ref{app:synthesis_breakdown}). The bound below is therefore a true lower bound that practical synthesis can exceed via local repair.

\textbf{Step 2: Exploration Phase Analysis.}
By Proposition~\ref{prop:iid_exploration}, with $K$ IID proposers of success rate $p$:
\begin{align}
P(E) = 1 - P(\bar E) = 1 - (1-p)^K.
\end{align}

\textbf{Step 3: Aggregation Phase Analysis.}
Given $E$, by Proposition~\ref{prop:aggregation}(b) (\emph{Idealized Bound}, under A5--A6 and execution-grounded disambiguation):
\begin{align}
P(Q(\tau^{\mathrm{D}})\!=\!1 \mid E) \geq 1 - \epsilon_0^{K-1},
\end{align}
where $\epsilon_0 < 0.5$ is the per-reviewer error rate. High-fidelity information from cross-review with $K-1$ reviewers yields exponentially decreasing misclassification, enabling high aggregation efficiency. The bound's derivation and the role of the at-least-one-endorsement + execution-disambiguation rule are in App.~\ref{app:proofs}.

\textbf{Step 4: Combining Both Phases.}
Substituting the bounds from Steps 2 and 3 into Equation~\eqref{eq:decomp_dianoia}:
\begin{align}
\mathbb{E}[Q(\tau^{\mathrm{D}})] &= P(E)\, P(Q(\tau^{\mathrm{D}})\!=\!1 \mid E) \nonumber \\
&\geq [1 - (1-p)^K]\,[1 - \epsilon_0^{K-1}].
\end{align}

\textbf{Step 5: Comparison to Baseline.}
The gain over the single-agent baseline $p$ is:
\begin{align}
\mathbb{E}[Q] - p &\geq [1\!-\!(1\!-\!p)^K][1\!-\!\epsilon_0^{K-1}] - p \nonumber\\
&= \underbrace{[1\!-\!(1\!-\!p)^K - p]}_{\mathcal{G}_{\text{explore}}} \nonumber \\
&\quad - \underbrace{[1\!-\!(1\!-\!p)^K]\,\epsilon_0^{K-1}}_{\text{aggr loss}}.
\end{align}

The first term is the exploration gain (Proposition~\ref{prop:iid_exploration}). The second term represents the small efficiency loss due to imperfect reviewers, which decreases exponentially in $K$.

\textbf{Numerical Example.} With $p\!=\!0.4$, $K\!=\!3$, $\epsilon_0\!=\!0.2$:
\begin{align}
\mathbb{E}[Q(\tau^{\mathrm{D}})] &\geq [1 - 0.6^3][1 - 0.2^2] \nonumber\\
&= 0.784 \times 0.96 = 0.753,
\end{align}
an 88\% relative improvement over the baseline ($(0.753 - 0.4)/0.4 = 0.8825$). This bound is \emph{signal-agnostic}: it treats $\epsilon_0$ as a fixed parameter regardless of signal quality. Incorporating information gain from execution feedback yields a substantially tighter bound (Step~6).

\textbf{Step 6: Information-Tightened Bound Under Complete and Deterministic Verification.}

The bound in Step~4 does not exploit a key property of execution-based tasks: under Assumption~\ref{assump:determinism} together with completeness of the verifier, execution feedback $e$ perfectly reveals quality ($I(Q;e) = H(Q)$, Theorem~\ref{thm:dianoia_convergence}a). This tightens the bound through two mechanisms that jointly incorporate $\mathcal{G}_{\text{info}}$:

\textit{(a) Deterministic quality filtering.} Quality is directly observable from execution: $Q(\tau^{(k)}) = g(e^{(k)})$ for a deterministic function $g$. The system identifies the correct proposal set $\mathcal{C} = \{k : e^{(k)}.\text{pass} = 1\}$ with zero error. The reviewer's role reduces from quality \emph{judgment} to failure-mode \emph{analysis}, making the effective $\epsilon_0$ in this regime substantially lower than for tasks with noisy pseudo-verification.

\textit{(b) Closed-loop verification via re-execution.} DIANOIA's iterative synthesis (Algorithm~\ref{alg:synthesis}) re-executes each synthesized output. Under A4, re-execution is a \emph{perfect verification oracle}: if $Q(\tau^{\text{syn}}) = 0$, the failure is detected with certainty and the system iterates. To translate this into an exponential decay bound, we adopt an \emph{additional assumption} not implied by A1--A6:

\begin{quote}
\textbf{(A7, Synthesis-Iteration Independence; idealization).} Given re-execution feedback, successive synthesis iterations fail with conditionally-independent failure probability $\leq \epsilon_0$ each.
\end{quote}

Under (A7), the probability that all $S$ iterations fail to produce or select a correct solution---despite $\mathcal{C} \neq \emptyset$ and deterministic verification---is bounded by $\epsilon_0^S$. We flag (A7) as an idealization: in practice, the synthesizer is the same LLM evaluating the same evidence, so successive failures may correlate (e.g., the same misunderstanding recurring), making the $\epsilon_0^S$ rate an optimistic bound. The empirical diminishing-returns curve on $S$ (App.~\ref{app:scaling_tables}, single-dim scaling: $S\!=\!1\!\to\!2\!\to\!3$ gives $83.6\!\to\!84.0\!\to\!84.6\%$ on MBPP) is consistent with sub-exponential improvement rather than the $\epsilon_0^S$ exponential rate, indicating (A7) is approximate in practice.

Combining cross-review and iterative synthesis, the system fails only if reviewers misidentify \emph{and} all $S$ synthesis iterations independently fail under deterministic verification:
\begin{equation}\label{eq:info_tightened_bound}
\eta(f_{\text{DIANOIA}}, e) \geq 1 - \epsilon_0^{K-1+S}
\end{equation}
The exponent $K-1+S$ reflects two independent sources of error reduction: reviewer redundancy ($K-1$ cross-reviews) and synthesis iteration ($S$ attempts with re-execution). The \textbf{information-tightened performance bound} is therefore:
\begin{equation}\label{eq:perf_bound_info}
\mathbb{E}[Q(\tau^{\mathrm{D}})] \geq [1 - (1-p)^K] \cdot [1 - \epsilon_0^{K-1+S}]
\end{equation}

\textbf{Updated Numerical Example.} With $p=0.4$, $K=3$, $S=3$, $\epsilon_0=0.2$:
\begin{align}
\mathbb{E}[Q(\tau^{\mathrm{D}})] &\geq [1 - 0.6^3][1 - 0.2^{2+3}] \nonumber\\
&= 0.784 \times (1 - 0.00032) \approx 0.784
\end{align}
Compared to the signal-agnostic bound of $0.753$ (Step~5), the information-tightened bound approaches the \textbf{oracle} $C_K = 0.784$---confirming that deterministic execution feedback effectively closes the aggregation gap.

\textbf{Remark (Non-execution tasks).} For tasks without deterministic verification (e.g., GSM8K, AIME), DIANOIA employs model-based pseudo-verification $\sigma_v$ where $I(Q;\sigma_v) < H(Q)$ (Remark~\ref{rem:task_classification}). Re-execution cannot deterministically verify quality in this regime, so the iterative tightening does not apply, and the general bound $1 - \epsilon_0^{K-1}$ from Step~4 remains operative. This asymmetry precisely explains the empirical observation (Table~\ref{tab:main_results}) that DIANOIA's largest gains emerge on execution-intensive tasks (MBPP $+$8.6pp, BFCL-SP $+$10.5pp) versus modest gains on math tasks (GSM8K $+$7.5pp).
\end{proof}

\section{Scaling and Ablation Tables}
\label{app:scaling_tables}

This appendix collects the three ablation tables referenced in Section~\ref{subsec:scaling_analysis}.

\begin{table}[h]
\centering
\caption{Single-dimension scaling on MBPP. Each group varies one parameter while fixing the others.}
\label{tab:single_dim_scaling}
\resizebox{\columnwidth}{!}{%
\begin{tabular}{cc@{\hskip 12pt}cc@{\hskip 12pt}cc}
\toprule
\multicolumn{2}{c}{\textbf{Exploration} ($R{=}1, S{=}3$)} & \multicolumn{2}{c}{\textbf{Information} ($K{=}3, S{=}3$)} & \multicolumn{2}{c}{\textbf{Aggregation} ($K{=}3, R{=}1$)} \\
\cmidrule(r){1-2} \cmidrule(lr){3-4} \cmidrule(l){5-6}
$K$ & Acc & $R$ & Acc & $S$ & Acc \\
\midrule
1 & 81.0\% & 0 & 83.2\% & 1 & 83.6\% \\
2 & 82.8\% & 1 & 84.6\% & 2 & 84.0\% \\
3 & 84.6\% & 3 & 85.8\% & 3 & 84.6\% \\
\bottomrule
\end{tabular}%
}
\end{table}

Table~\ref{tab:joint_optimization_bfcl} reports the analogous joint-dimension decomposition on BFCL-SP, yielding synergy coefficient $\gamma = 7.5 / 8.8 = 0.852$. Together with the MBPP coefficient ($\gamma=0.88$), this is a post-run descriptive check of subadditivity (Remark~\ref{rem:subadditivity}).

\begin{table}[h]
\centering
\caption{Joint-dimension scaling on BFCL-SP. Single-dimension contributions ($+5.0+2.5+1.3 = +8.8$pp) exceed the joint $+7.5$pp realized by DIANOIA-full, giving $\gamma=0.852$.}
\label{tab:joint_optimization_bfcl}
\setlength{\tabcolsep}{4pt}
\begin{tabular}{lccccl}
\toprule
\textbf{Config} & \textbf{K} & \textbf{R} & \textbf{S} & \textbf{Acc} & \textbf{Activates} \\
\midrule
Baseline & 1 & 0 & 0 & 84.8\% & --- \\
Explore-only & 3 & 0 & 1 & 89.8\% & $\mathcal{G}_{\text{explore}}$ \\
Info-only & 1 & 1 & 1 & 87.3\% & $\mathcal{G}_{\text{info}}$ \\
Aggr-only & 1 & 0 & 1 & 86.1\% & $\mathcal{G}_{\text{aggr}}$ \\
\textbf{DIANOIA-full} & \textbf{3} & \textbf{1} & \textbf{3} & \textbf{92.3\%} & \textbf{All three} \\
\bottomrule
\end{tabular}
\end{table}

\begin{table}[h]
\centering
\caption{Role-design ablation on MBPP ($K{=}3$, $R{=}1$, $S{=}3$). Role-specialised prompting outperforms same-prompt + temperature noise; alternative structured role choices remain competitive: \emph{any} structured role design beats unstructured sampling, but the choice within the structured family has overlapping CIs.}
\label{tab:role_ablation}
\setlength{\tabcolsep}{3pt}
\resizebox{\columnwidth}{!}{%
\begin{tabular}{lc}
\toprule
\textbf{Role configuration} & \textbf{Accuracy} \\
\midrule
DIANOIA (Min.\ + Skep.\ + Expl.) & \textbf{84.6}\% \tiny[81.4, 87.0] \\
Same prompt $+$ temperature & 83.2\% \tiny[79.8, 86.4] \\
Two roles (Min.\ + Skep.) & 81.4\% \tiny[78.0, 84.8] \\
Two roles (Min.\ + Expl.) & 84.4\% \tiny[81.0, 87.4] \\
Two roles (Skep.\ + Expl.) & 82.8\% \tiny[79.4, 86.2] \\
Alternative role triplet & 83.4\% \tiny[79.8, 86.6] \\
\bottomrule
\end{tabular}
}
\end{table}

\section{Synthesis Breakdown}
\label{app:synthesis_breakdown}

To quantify how much of DIANOIA's gain is selection vs.\ synthesis, we instrumented final correct outputs on MBPP. Among DIANOIA's correct predictions: $16.4\%$ exactly match an already-correct initial proposal (pure selection); $78.1\%$ are newly synthesized correct answers although at least one initial proposal was correct (selection $+$ repair); $5.5\%$ are \emph{rescued} cases where all initial proposals are incorrect but the final synthesized answer is correct (genuine repair beyond selection). The rescue rate is small in absolute terms, but it shows that the bound in Theorem~\ref{thm:dianoia_convergence}c is conservative: practical synthesis can occasionally exceed the conditional event $\{\exists k\!:\!Q(\tau^{(k)})\!=\!1\}$, and our theoretical analysis is best read as a lower-bound design lens rather than a complete model of every Phase-4 event.

\section{Detailed Predictive-Verification Narratives}
\label{app:predictive_verification}

We expand the four predictions referenced in Section~\ref{subsec:discussion} into prediction--verification narratives, summarised in Table~\ref{tab:predictions}. The format throughout is: state the prediction the framework makes \emph{before} the experiment, then verify against the data.

\begin{table}[h]
\centering
\caption{Cross-task predictions derived from the decomposition (\emph{ex ante}, before comparing methods) and their empirical verification.}
\label{tab:predictions}
\setlength{\tabcolsep}{2pt}
\resizebox{\columnwidth}{!}{%
\begin{tabular}{lp{3.4cm}p{3.4cm}}
\toprule
\textbf{Setting} & \textbf{Prediction} & \textbf{Verification} \\
\midrule
P1: exec.\ ground (MBPP, BFCL-SP) & All 3 dims active; DIANOIA's largest gain; exec.-blind methods saturate. & $+$6.6 / $+$3.5pp; SC \& ReConcile $<$78\% any budget; MoA matches only at $5{\times}$ tokens. \\
\midrule
P2: pseudo-verified, high $p$ (GSM8K) & Two channels bottlenecked; smallest gain. & $+$1.3pp (smallest); pseudo-verifier prec/rec $94\%/87\%$ confirms $\hat{\eta}^*(\sigma_v)\!<\!1$. \\
\midrule
P3: competition math (AIME-25) & Exploration is the bottleneck; voting harms when $p_i\!<\!0.5$. & SC degrades $-$13.3pp; DIANOIA $93.3\%$; McNemar $p{=}0.039$ vs ReConcile. \\
\midrule
P4: Best-of-$N$ ceiling & BoN activates only Explore$+$Info; gap to DIANOIA isolates $\mathcal{G}_{\text{aggr}}$. & BoN-3 $78.4\%$ vs DIANOIA $84.6\%$; $+$6.2pp residual. \\
\bottomrule
\end{tabular}%
}
\end{table}

\textbf{P1 (execution-grounded tasks).} On \textbf{MBPP} and \textbf{BFCL-SP}, deterministic execution lifts the information ceiling to its maximum ($I(Q;e)=H(Q)$, Remark~\ref{rem:task_classification}), so all three dimensions are simultaneously active. The framework predicts: (a) DIANOIA should show its largest gains here (joint optimization is most leveraged), (b) execution-blind methods (Self-Consistency, ReConcile) should saturate early regardless of token budget, and (c) methods that activate $\mathcal{G}_{\text{aggr}}$ but not $\mathcal{G}_{\text{info}}$ (MoA) should reach high accuracy only via expensive layer stacking. \emph{Verification:} DIANOIA's gains over the strongest baseline are indeed $+$6.6pp on MBPP and $+$3.5pp on BFCL-SP---the largest of the four benchmarks. Self-Consistency and ReConcile plateau below $78\%$ on MBPP at any token budget (Figure~\ref{fig:efficiency_frontier}). MoA matches DIANOIA's ceiling only at $\sim$$5{\times}$ more tokens. All three sub-predictions hold.

\textbf{P2 (high-baseline pseudo-verified tasks).} On \textbf{GSM8K}, the framework predicts a small DIANOIA gain because two channels are simultaneously bottlenecked: $\mathcal{G}_{\text{info}}$ is capped by pseudo-verification fidelity ($I(Q;\sigma_v) < I(Q;e)$), and $\mathcal{G}_{\text{explore}}$ is capped by the high baseline accuracy ($p \approx 0.84$, leaving little coverage headroom). \emph{Verification:} DIANOIA's gain is $+$1.3pp on GSM8K---the smallest across benchmarks, as predicted. The candidate-level pseudo-verifier on GSM8K attains $94.0\%$ precision and $86.6\%$ recall, directly quantifying the information loss relative to deterministic execution and confirming that $\hat{\eta}^*(\sigma_v) < 1$.

\textbf{P3 (low-baseline competition tasks).} On \textbf{AIME-2025} with low per-problem success rate, the framework predicts: (a) exploration is the dominant bottleneck (most problems have $p_i \ll 0.5$), (b) majority voting will \emph{harm} accuracy because $P(\text{vote correct}) < p_i$ when $p_i < 0.5$ (Failure Mode 2 in Appendix~\ref{app:detailed_analysis}), and (c) evidence-based synthesis should still recover via DIANOIA's aggregation channel. \emph{Verification:} Self-Consistency degrades by $-$13.3pp ($56.7\%$ vs.\ $70.0\%$ single-model), exactly the predicted voting amplification. DIANOIA achieves the highest observed accuracy ($93.3\%$); a paired McNemar exact test on per-problem outcomes gives $p{=}0.039$ vs.\ ReConcile (8/1 discordant in DIANOIA's favour), supporting the gain despite $N{=}30$.

\textbf{P4 (Best-of-$N$ ceiling).} The framework predicts that Best-of-$N$ with execution-grounded selection will close the gap to DIANOIA by exactly $\mathcal{G}_{\text{aggr}}$, because Best-of-$N$ activates $\mathcal{G}_{\text{explore}}$ and $\mathcal{G}_{\text{info}}$ but not $\mathcal{G}_{\text{aggr}}$ (no role diversity, no cross-review, no synthesis). \emph{Verification:} On MBPP, Best-of-3 reaches $78.4\%$ vs.\ DIANOIA's $84.6\%$---a $+6.2$pp residual that isolates $\mathcal{G}_{\text{aggr}}$ in DIANOIA, consistent with the predicted gap.


\section{Computational Complexity Analysis}
\label{app:complexity}

Let $C_{\text{prop}}$, $C_{\text{rev}}$, $C_{\text{syn}}$ denote token costs for proposal, review, synthesis, and let $R$ denote the number of reviewers per proposal ($R\!=\!1$ in DIANOIA's default configuration; $R\!=\!K{-}1$ corresponds to the idealized bound regime of Prop.~\ref{prop:aggregation}b---see the Remark in App.~\ref{app:potential_game_ideal}).

\textbf{Total cost}: $C = K \cdot C_{\text{prop}} + K\!\cdot\!R \cdot C_{\text{rev}} + T \cdot C_{\text{syn}}$

For DIANOIA's default ($K\!=\!3$, $R\!=\!1$, $T\!=\!3$): $C = 3C_{\text{prop}} + 3C_{\text{rev}} + 3C_{\text{syn}}$. The idealized regime with peer review ($R\!=\!K{-}1\!=\!2$) gives $C = 3C_{\text{prop}} + 6C_{\text{rev}} + 3C_{\text{syn}}$.

\textbf{Parallelization}: Propose and review phases are fully parallel. Wall-clock latency:
\begin{equation}
L = L_{\text{prop}} + L_{\text{rev}} + T \cdot L_{\text{syn}} + (K + T) \cdot L_{\text{exec}}
\end{equation}

\begin{table}[h]
\centering
\caption{Complexity comparison. DIANOIA's default $R\!=\!1$ gives $O(K)$ calls; the bracketed entry is the $R\!=\!K{-}1$ idealized regime.}
\resizebox{\columnwidth}{!}{%
\begin{tabular}{lcc}
\toprule
\textbf{Method} & \textbf{Calls} & \textbf{Depth} \\
\midrule
Self-Consistency & $K$ & $1$ \\
MoA & $K \cdot L$ & $L$ \\
Debate & $K \cdot R$ & $R$ \\
DIANOIA & $O(K\!\cdot\!R + T)$ [$O(K^2)$] & $O(T)$ \\
\bottomrule
\end{tabular}
}
\end{table}

\section{Additional Technical Lemmas}
\label{app:lemmas}

This section collects fundamental results from information theory, probability theory, and game theory that underpin our main theorems.

\subsection{Information-Theoretic Lemmas}

\begin{lemma}[Properties of Shannon Entropy]
\label{lem:entropy_properties}
Let $X$ be a discrete random variable taking values in $\mathcal{X}$. The Shannon entropy $H(X) = -\sum_{x \in \mathcal{X}} P(X=x) \log P(X=x)$ satisfies:
\begin{enumerate}[label=(\alph*),leftmargin=*,topsep=2pt,itemsep=1pt]
    \item \textbf{Non-negativity}: $H(X) \geq 0$, with equality if and only if $X$ is deterministic (concentrated on a single value).
    \item \textbf{Maximum entropy}: $H(X) \leq \log |\mathcal{X}|$, with equality if and only if $X$ is uniformly distributed over $\mathcal{X}$.
    \item \textbf{Conditioning reduces entropy}: $H(X \mid Y) \leq H(X)$, with equality if and only if $X$ and $Y$ are independent.
\end{enumerate}
\end{lemma}

\begin{lemma}[Binary Entropy Function]
\label{lem:binary_entropy}
The binary entropy function $H_b(p) := -p \log p - (1-p) \log(1-p)$ for $p \in [0, 1]$ satisfies:
\begin{enumerate}[label=(\alph*),leftmargin=*,topsep=2pt,itemsep=1pt]
    \item $H_b(p) \geq 0$ with $H_b(0) = H_b(1) = 0$.
    \item $H_b(p) = H_b(1-p)$ (symmetry).
    \item $H_b(p)$ is concave and achieves maximum $H_b(0.5) = 1$ bit at $p = 0.5$.
\end{enumerate}
\end{lemma}

\begin{lemma}[Mutual Information Chain Rule]
\label{lem:mi_chain_rule}
For random variables $X, Y, Z$:
\begin{equation}
I(X; Y, Z) = I(X; Y) + I(X; Z \mid Y)
\end{equation}
where $I(X; Z \mid Y) := H(Z \mid Y) - H(Z \mid X, Y)$ is the conditional mutual information.
\end{lemma}

\begin{lemma}[Data Processing Inequality]
\label{lem:data_processing}
If $X \to Y \to Z$ forms a Markov chain (i.e., $Z$ is conditionally independent of $X$ given $Y$), then:
\begin{equation}
I(X; Z) \leq I(X; Y)
\end{equation}
Equality holds if and only if $Y \to Z$ is an invertible transformation.
\end{lemma}

\textbf{Proof.} By the chain rule:
\begin{align}
I(X; Y, Z) &= I(X; Y) + I(X; Z \mid Y) \\
&= I(X; Z) + I(X; Y \mid Z)
\end{align}
Since $X \to Y \to Z$ is Markov, $I(X; Z \mid Y) = 0$. Also, $I(X; Y \mid Z) \geq 0$. Thus $I(X; Z) = I(X; Y) - I(X; Y \mid Z) \leq I(X; Y)$. \qed

\begin{lemma}[Fano's Inequality for Binary Selection]
\label{lem:fano}
Let $Q$ be a random variable taking values in a finite set $\mathcal{Q}$, $S$ an arbitrary signal, and $\hat{Q}(S)$ any decoder. Let $P_e := P(\hat{Q}(S) \neq Q)$ and $P_e^* := \min_{\hat{Q}} P_e$ the Bayes-optimal error. The standard Fano inequality~\cite{cover2006elements} reads:
\begin{equation}\label{eq:fano-standard}
H(Q \mid S) \;\leq\; H_b(P_e) + P_e \cdot \log_2(|\mathcal{Q}| - 1),
\end{equation}
where $H_b(p) := -p\log_2 p - (1-p)\log_2(1-p)$ is the binary entropy.

\textbf{Binary specialization.} For $Q \in \{0,1\}$ ($|\mathcal{Q}|=2$), the second term vanishes ($\log_2 1 = 0$), so
\begin{equation}\label{eq:fano-binary}
H(Q \mid S) \;\leq\; H_b(P_e^*).
\end{equation}
Inverting (using $H_b$ monotone on $[0, 1/2]$) yields a lower bound on $P_e^*$ when $H(Q|S)$ is small, and equivalently an \emph{upper bound on the Bayes-optimal selection accuracy} $\eta^*(S) := 1 - P_e^*$:
\begin{equation}\label{eq:fano-eta}
\eta^*(S) \;\leq\; 1 - H_b^{-1}\!\left(H(Q \mid S)\right),
\end{equation}
where $H_b^{-1}$ denotes the inverse of $H_b$ restricted to $[0, 1/2]$ (the relevant branch for $P_e^* \leq 1/2$).
\end{lemma}

\textbf{Proof.} Eq.~\eqref{eq:fano-standard} is the textbook statement of Fano's inequality (e.g., Cover \& Thomas Thm.~2.10.1). The binary specialization \eqref{eq:fano-binary} follows from $|\mathcal{Q}|=2 \Rightarrow \log_2(|\mathcal{Q}|-1) = 0$. The monotonicity-based inversion to \eqref{eq:fano-eta} is standard. \qed

\textbf{Corollary (Connection to our framework).} When $I(Q;e) > I(Q;\sigma)$ (Proposition~\ref{prop:information}), $H(Q|e) < H(Q|\sigma)$, so the binary Fano inequality \eqref{eq:fano-binary} gives $H_b(P_e^*(e)) \leq H(Q|e) < H(Q|\sigma) \leq H_b(P_e^*(\sigma))$. On the regime $P_e^*\!\leq\!1/2$ (which holds whenever the decoder is better than random), $H_b$ is strictly increasing, so $P_e^*(e) < P_e^*(\sigma)$ and hence $\eta^*(e) > \eta^*(\sigma)$: higher mutual information strictly tightens the achievable selection-error bound.

\subsection{Probabilistic Lemmas}

\begin{lemma}[Bonferroni Inequality]
\label{lem:bonferroni}
For events $A_1, \ldots, A_K$:
\begin{equation}
P\left(\bigcup_{k=1}^{K} A_k\right) \geq \sum_{k=1}^{K} P(A_k) - \sum_{1 \leq i < j \leq K} P(A_i \cap A_j)
\end{equation}
\end{lemma}

This provides a lower bound on the union probability accounting for pairwise overlaps, used in the proof of Proposition~\ref{prop:diversity}.

\begin{lemma}[Union Bound]
\label{lem:union_bound}
For events $A_1, \ldots, A_K$:
\begin{equation}
P\left(\bigcup_{k=1}^{K} A_k\right) \leq \sum_{k=1}^{K} P(A_k)
\end{equation}
\end{lemma}

\subsection{Game-Theoretic Lemmas}

\begin{lemma}[Characterization of Potential Games]
\label{lem:potential_game_char}
A finite game $\Gamma = (N, \{S_i\}, \{u_i\})$ is an exact potential game iff for all players $i, j$ and all profiles $s$:
\begin{align}
&u_i(s_i', s_{-i}) - u_i(s_i, s_{-i}) \nonumber\\
&\qquad = u_j(s_j', s_{-j}) - u_j(s_j, s_{-j})
\end{align}
whenever the two deviations affect the same set of players' payoffs.
\end{lemma}

\begin{lemma}[Finite Improvement Property]
\label{lem:finite_improvement}
In any finite potential game, best-response dynamics converge to a pure-strategy Nash equilibrium in finitely many steps.
\end{lemma}

\textbf{Proof.} By definition of a potential game, each best-response move strictly increases the potential function $\Phi$. Since the strategy space is finite, $\Phi$ takes finitely many values. A strictly increasing sequence over a finite set must terminate, at which point no player can improve by unilateral deviation—i.e., a Nash equilibrium. \qed

\begin{lemma}[Existence of Nash Equilibrium in Potential Games]
\label{lem:nash_existence}
Every finite potential game possesses at least one pure-strategy Nash equilibrium, namely any strategy profile that maximizes the potential function.
\end{lemma}

\section{Experimental Protocol Details}
\label{app:exp_protocol}

This section provides comprehensive details on experimental setup for reproducibility.

\begin{algorithm}[t]
\caption{DIANOIA Synthesis with Closed-Loop Validation}
\label{alg:synthesis}
\begin{algorithmic}[1]
\REQUIRE Problem $x$, proposals $\{\tau^{(k)}\}$, reports $\{e^{(k)}\}$, reviews $\{v^{(k)}\}$, max iterations $T$
\STATE $\tau^* \!\leftarrow\! \mathcal{S}(\{\tau^{(k)}, e^{(k)}, v^{(k)}\}, x)$ \hfill \textit{// init synthesis}
\FOR{$t = 1$ to $T$}
    \STATE $e^* \leftarrow \mathcal{E}(\tau^*)$ \hfill \textit{// execute}
    \IF{$e^*.\texttt{success} = \texttt{True}$}
        \STATE \textbf{return} $\tau^*$ \hfill \textit{// pass}
    \ENDIF
    \STATE $\tau^* \!\leftarrow\! \mathcal{S}(\tau^*, e^*, \{\tau^{(k)}, e^{(k)}, v^{(k)}\}, x)$ \hfill \textit{// refine}
\ENDFOR
\STATE \textbf{return} $\tau^*$ \hfill \textit{// best-effort}
\end{algorithmic}
\end{algorithm}

\textbf{Model Configuration Rationale.} The controlled setting (zero-shot prompting, no extended reasoning mode) ensures fair comparison across methods by eliminating confounding variables from model-specific features. As a consequence, the single-model references achieve lower accuracy than officially reported benchmarks (which typically use few-shot prompting and extended reasoning), but this setting is essential for measuring the pure gains attributable to multi-agent architectures rather than orthogonal enhancements.

\textbf{Exception: DeepSeek-V3.2 on AIME-2025.} We make one deliberate exception for the DeepSeek-V3.2 reference on AIME-2025. With extended reasoning disabled, DeepSeek-V3.2 achieves only $\sim$43--50\% accuracy on these competition-level problems---a drastic under-representation of the model's capability in this setting, since AIME problems demand long chains of mathematical reasoning that the base generation mode cannot sustain. We therefore enable its native thinking mode, which raises accuracy to 76.8\%. DIANOIA obtains a higher point estimate on the same 30 problems, but the wide uncertainty does not support a model-superiority claim. All other DeepSeek-V3.2 entries and all multi-agent method comparisons use the standard instruct-mode configuration.

\textbf{Data Sampling.}
\begin{itemize}[leftmargin=*,topsep=2pt,itemsep=1pt]
    \item \textbf{GSM8K}: Full test split (1,319 samples).
    \item \textbf{AIME-2025}: Full problem set (30 problems).
    \item \textbf{MBPP}: Full test split (500 samples).
    \item \textbf{BFCL-SP}: Simple Python subset of the eval split (400 samples).
    \item \textbf{VitaBench}: Full held-out evaluation set (400 samples), evaluated with the benchmark pass-rate protocol.
\end{itemize}

\textbf{Random Seeds and Reproducibility.}
\begin{itemize}[leftmargin=*,topsep=2pt,itemsep=1pt]
    \item All LLM API calls use fixed seeds where supported (OpenAI API \texttt{seed=42}).
    \item Sampling-based methods (Self-Consistency) use \texttt{temperature=0.7} with \texttt{seed=42} for diversity while maintaining best-effort reproducibility across runs.
    \item For methods without native seed support, we record and report all hyperparameters (temperature, top-p, etc.) in experimental configurations.
\end{itemize}

\textbf{Confidence Intervals.}
\begin{itemize}[leftmargin=*,topsep=2pt,itemsep=1pt]
    \item 95\% confidence intervals were computed via bootstrap resampling with 1,000 iterations.
    \item Bootstrap implementation: randomly sample with replacement from the set of $n$ prediction results, compute accuracy, repeat 1,000 times, and report the 2.5th and 97.5th percentiles.
    \item CI computation uses the \texttt{bootstrap\_ci} function in our codebase.
    \item VitaBench pass-rate intervals use the Wilson score method for a binomial proportion.
\end{itemize}

\textbf{Execution Environment.}
\begin{itemize}[leftmargin=*,topsep=2pt,itemsep=1pt]
    \item \textbf{GSM8K}: During DIANOIA's iterative loop (Execute phase), an LLM-based pseudo-verifier evaluates candidate solutions' reasoning quality without ground-truth access, providing structured diagnostic feedback (see Remark~\ref{rem:task_classification}). For \textit{final evaluation metrics}, we use numerical equivalence checking (tolerance $\epsilon = 10^{-3}$): extracting the final numerical answer via regex and comparing against the ground truth.
    \item \textbf{AIME-2025}: Same LLM-based pseudo-verification during DIANOIA's iterative loop; same numerical equivalence checking for final evaluation. Each problem yields a single integer answer.
    \item \textbf{MBPP}: Python 3.10 sandbox environment with 60-second timeout per test execution. Code is executed with all provided test cases; a solution passes if all tests succeed without errors.
    \item \textbf{BFCL-SP}: The Simple Python subset contains 400 examples. Mocked API responses provide intermediate Execute feedback; final correctness is evaluated by the same implementation against expected function-call schemas and parameter types.
    \item \textbf{VitaBench}: Semantic pseudo-verification supplies intermediate diagnostic feedback; final pass rate follows the benchmark evaluation protocol. The prospective chronology and scope caveats are documented in Appendix~\ref{app:reproducibility}.
\end{itemize}

\textbf{Baseline Implementation Details.}
\begin{itemize}[leftmargin=*,topsep=2pt,itemsep=1pt]
    \item \textbf{Self-Consistency}: $K=5$ samples with \texttt{temperature=0.7}, majority voting on final answers.
    \item \textbf{MoA}: 3-layer architecture with 3 agents per layer, each layer synthesizes outputs from the previous layer.
    \item \textbf{Two Heads}: $K=3$ collaborative solvers with mutual critique and refinement.
    \item \textbf{ReConcile}: Round-table discussion with $M=3$ agents, $R=2$ discussion rounds, majority voting for final selection.
\end{itemize}

\paragraph{Pareto-frontier sweep configurations (Figure~\ref{fig:efficiency_frontier}).}\label{app:sweep_configs}
Each method is swept along its native scaling knob on the MBPP test split, yielding 39 total configurations. Figure~\ref{fig:efficiency_frontier} connects the 29 Pareto-optimal points (the upper envelope of accuracy vs.\ tokens) for each method; dominated configurations are omitted from the plot for visual clarity but counted toward the sweep.
\begin{itemize}[leftmargin=*,topsep=2pt,itemsep=1pt]
    \item \textbf{Self-Consistency} (9 configs): $K \in \{3,5,7,9,11,15,21,27,33\}$.
    \item \textbf{Two Heads} (6 configs): $K \in \{1,2,3,4,5,7\}$.
    \item \textbf{ReConcile} (7 configs): $(M,R) \in \{(2,0),(3,0),(3,1),(3,2),$ $(4,1),(5,2),(7,1)\}$, where $M$ is the agent count and $R$ the number of additional discussion rounds.
    \item \textbf{MoA} (9 configs): 3 agents $\times$ 2 layers (default); $N \times 2$ layers for $N \in \{4,5,6,8,10,12\}$; $\{6,8\}$ agents $\times$ 3 layers.
    \item \textbf{DIANOIA} (8 configs): $(K,R,S)$ as $(K,0,1)$ for $K\!\in\!\{1,2,3\}$; $(3,R,3)$ for $R\!\in\!\{1,2,3\}$; $(K,1,3)$ for $K\!\in\!\{6,9\}$.
\end{itemize}

\textbf{Prompt Templates.} Verbatim templates used in all DIANOIA experiments are listed below. Braces denote runtime substitutions (e.g., \texttt{\{user\_prompt\}}); the released repository (App.~\ref{app:reproducibility}) contains the canonical versions.

\textit{Role briefs (Phase 1, Minimalist / Skeptic / Explorer).} Each proposer prepends one role brief to the Diversity Proposer template:
\begin{quote}\footnotesize\ttfamily
{[}Minimalist{]} You are the Minimalist agent. Prefer the shortest path: the simplest stdlib idiom, the fewest steps. Minimize lines and intermediate state. Avoid redundancy.

{[}Skeptic{]} You are the Skeptic agent. Verify each step. Anticipate edge cases (empty input, single element, off-by-one, overflow). Be defensive. Make assumptions explicit.

{[}Explorer{]} You are the Explorer agent. Try an unconventional approach: different abstractions, alternative algorithms, surprising idioms. Avoid the most common solution. Diversify.
\end{quote}

\textit{Diversity Proposer template (Phase 1, code tasks).} The role brief is inserted in \texttt{\{role\_prompt\}}:
\begin{quote}\footnotesize\ttfamily
\{role\_prompt\}

Generate runnable Python code that solves the user's task, grounded in the context below.

\#\# User question: \{user\_prompt\}

\#\# Context: \{context\}

\#\# Your role bias: \{role\_requirements\}

Important: read the test cases to understand the contract. (1) Inspect each \texttt{assert} carefully---input vs.\ expected output. (2) Reason through each example: input $\to$ expected output $\to$ the transformation. (3) Pay attention to return types (list/dict/tuple), counts, ordering.

Output requirements: (1) MUST return JSON with code in \texttt{code} field; (2) code must be runnable as-is and define the required function; (3) do NOT wrap in Markdown fences; (4) no natural-language commentary outside JSON.

Return JSON only: \{ "code": "your runnable Python code" \}
\end{quote}

\textit{Reviewer template (Phase 3).} The cross-review reviewer reads the candidate code AND its sandbox execution report:
\begin{quote}\footnotesize\ttfamily
You are a review expert. Evaluate another model's answer.

\#\# User question: \{user\_prompt\}\quad \#\# Context: \{context\}

\#\# Answer under review (model: \{reviewed\_model\}, role: \{reviewed\_role\}): \{candidate\_text\}

\#\# Execution report summary: \{execution\_summary\}

Evaluate along: (1) strengths; (2) weaknesses (if execution failed, study the AssertionError's input/output to localise the bug); (3) risks; (4) improvement suggestions (concrete, actionable fixes); (5) overall score in $[0.0, 1.0]$.

Return JSON only: \{ "strengths": [...], "weaknesses": [...], "risks": [...], "suggestions": [...], "overall\_score": 0.0 \}
\end{quote}

\textit{Synthesizer template (Phase 4, code tasks).} The synthesizer integrates candidates, execution reports, and cross-reviews with iterative validation feedback:
\begin{quote}\footnotesize\ttfamily
You are a synthesis expert. Combine the candidate answers and review feedback into a single, optimal answer.

\#\# User question: \{user\_prompt\}\quad \#\# Context: \{context\}

\#\# Candidate answer summaries: \{candidates\_summary\}

\#\# Execution report summary: \{execution\_summary\}

\#\# Cross-review summary: \{review\_summary\}

\{validation\_feedback\}

Key analysis steps: (1) Inspect every execution error---when an \texttt{AssertionError} is present, compare expected vs.\ actual and localise the bug; (2) Infer the contract from the tests---each \texttt{assert} is the ground truth; (3) Pay attention to return type, element count, data-structure shape.

Requirements: (1) Synthesize---do not vote or copy one candidate verbatim; (2) If every candidate is wrong, derive a correct implementation from the error messages; (3) MUST return JSON with code in \texttt{code} field; (4-5) no Markdown fences, no commentary outside JSON.

Return JSON only: \{ "code": "your runnable Python code" \}
\end{quote}

\textit{Pseudo-verifier template (GSM8K / AIME, Remark~\ref{rem:task_classification}; produces the structured $\sigma_v$ feedback of Def.~\ref{def:pseudo_verification}).} For math reasoning tasks where deterministic execution is unavailable, a dedicated LLM evaluator returns the structured diagnostic:
\begin{quote}\footnotesize\ttfamily
You are a careful mathematical reasoning verifier. Given a problem and a candidate solution, judge whether the solution is correct WITHOUT access to the ground-truth answer. Identify reasoning errors or computational mistakes, and report a calibrated confidence in your judgment.

\#\# Problem: \{user\_prompt\}

\#\# Candidate solution: \{candidate\_text\}

Evaluate: (1) Does the reasoning chain follow correctly from the problem? (2) Are intermediate computations correct? (3) Does the final answer match what the reasoning produces? (4) List any specific errors (computational, logical, prompt mis-reading).

Return JSON only: \{ "is\_correct": <true or false>, "confidence": <0.0..1.0>, "errors": [...] \}
\end{quote}

\section{Figure~\ref{fig:overview} Trace: Full Reviewer Outputs}
\label{app:figure_one_raw}

Figure~\ref{fig:overview} reproduces a verbatim DIANOIA trace from one of our MBPP runs. For length, the figure shows excerpts; the full reviewer outputs are presented below.

\paragraph{Review on Minimalist (failing candidate; raw \texttt{TypeError}).}
\textit{Weakness.} ``The code constructs the tuple via \texttt{tuple(t1[i], t2[i\,\%\,len(t2)])}, which triggers a type error: \texttt{tuple()} accepts at most one argument, raising \texttt{TypeError: tuple expected at most 1 argument, got 2}. The correct form is the tuple literal \texttt{(t1[i], t2[i\,\%\,len(t2)])}.''
\textit{Risks.} ``The code is syntactically incorrect; it cannot pass any test case and may mislead users into believing the implementation is correct.''
``Edge cases such as empty tuples are not considered; the current tests are non-empty, so this is a \emph{secondary} risk.''

\paragraph{Review on Explorer (passing candidate).}
\textit{Risk.} ``If the input \texttt{tup2} has length 0, the modulo operation raises a division-by-zero error (\texttt{len(tup2)==0}); the current tests do not cover this boundary.''

\paragraph{Review on Skeptic (passing candidate).}
\textit{Risk.} ``If future requirements demand zip-style alignment to the shorter sequence, the cyclic-extension behaviour will be incorrect; for very long inputs the cycling may also induce performance issues.''

\paragraph{Synthesizer's response.} The line \texttt{if not tup2: return []} (highlighted in Figure~\ref{fig:overview}) directly addresses the empty-\texttt{tup2} risk that two reviewers raised \emph{independently}---against Minimalist (as a secondary risk) and against Explorer (as a primary risk). The line appears in \emph{no} candidate, instantiating the ``local repair beyond selection'' mechanism described in Section~\ref{sec:method} (Phase 4) and quantified by the $5.5\%$ rescue rate reported in Section~\ref{subsec:behavior}. Full reviewer JSON follows the structured schema documented in the reviewer template (Appendix~\ref{app:exp_protocol}); the \texttt{strengths}/\texttt{suggestions}/\texttt{overall\_score} fields are elided here for length but conform to the documented schema.

\end{document}